\pdfoutput=1

\documentclass[11pt]{article}

\usepackage{EMNLP2022}

\usepackage[ruled,vlined,linesnumbered]{algorithm2e}
\usepackage{amsfonts}
\usepackage{amsmath}
\usepackage{bbm}
\usepackage{caption}
\usepackage{graphicx}
\usepackage{hyperref}
\usepackage{latexsym}
\usepackage{subcaption}
\usepackage{times}
\usepackage{amsthm}

\usepackage[T1]{fontenc}

\usepackage[utf8]{inputenc}

\usepackage{microtype}

\usepackage{inconsolata}

\newtheorem{theorem}{Theorem}

\title{Discovering Low-rank Subspaces for Language-agnostic \\ Multilingual Representations}

\author{Zhihui Xie$^1$ \quad Handong Zhao$^2$ \quad Tong Yu$^2$ \quad Shuai Li$^1$\thanks{~~Corresponding author.} \\
  $^1$Shanghai Jiao Tong University \quad $^2$Adobe Research \\
  \texttt{\{fffffarmer,shuaili8\}@sjtu.edu.cn} \\
  \texttt{\{hazhao,tyu\}@adobe.com}
}

\makeatletter
\def\thickhline{%
  \noalign{\ifnum0=`}\fi\hrule \@height \thickarrayrulewidth \futurelet
   \reserved@a\@xthickhline}
\def\@xthickhline{\ifx\reserved@a\thickhline
               \vskip\doublerulesep
               \vskip-\thickarrayrulewidth
             \fi
      \ifnum0=`{\fi}}
\makeatother

\newlength{\thickarrayrulewidth}
\begin{document}
\maketitle

\newcommand{\zhihui}[1]{{\color{green}[#1]}}
\newcommand{\tong}[1]{{\color{blue}[tong: #1]}}

\newcommand{\original}{Original}
\newcommand{\demean}{Centered}
\newcommand{\lir}{LIR}
\newcommand{\ours}{LSAR}
\begin{abstract}
Large pretrained multilingual language models (ML-LMs) have shown remarkable capabilities of zero-shot cross-lingual transfer, without direct cross-lingual supervision.
While these results are promising, follow-up works found that, within the multilingual embedding spaces, there exists strong language identity information which hinders the expression of linguistic factors shared across languages.
For semantic tasks like cross-lingual sentence retrieval, it is desired to remove such language identity signals to fully leverage semantic information.
In this work, we provide a novel view of projecting away language-specific factors from a multilingual embedding space.
Specifically, we discover that there exists a low-rank subspace that primarily encodes information irrelevant to semantics (e.g., syntactic information).
To identify this subspace, we present a simple but effective \textit{unsupervised} method based on singular value decomposition with multiple monolingual corpora as input.
Once the subspace is found, we can directly project the original embeddings into the null space to boost language agnosticism without finetuning.
We systematically evaluate our method on various tasks including the challenging language-agnostic QA retrieval task.
Empirical results show that applying our method consistently leads to improvements over commonly used ML-LMs.



    
    
    
\end{abstract}
\section{Introduction}
Large language models pretrained with self-supervised objectives (e.g., masked language modeling) have become the \textit{de-facto} standard for various NLP tasks~\citep{peters-etal-2018-deep,devlin-etal-2019-bert,liu2019roberta}.
Follow-up extensions to the multilingual setting inherit similar training objectives and show very promising results~\citep{NEURIPS2019_c04c19c2,conneau-etal-2020-emerging,K2020Cross-Lingual}.
Despite these models are trained without explicit cross-lingual signals (i.e., translation pairs), they surprisingly exhibit impressive zero-shot cross-lingual transferability on natural language inference~\citep{conneau-etal-2018-xnli}, question answering~\citep{lewis-etal-2020-mlqa}, sentence retrieval~\citep{artetxe-schwenk-2019-massively}, etc.

While these ML-LMs offer practical solutions for cross-lingual tasks, there is an enduring debate about why the ML-LMs work.
From a positive perspective, \citet{pires-etal-2019-multilingual} conduct an exploratory study on mBERT~\citep{devlin-etal-2019-bert}, suggesting that cross-lingual transfer is possible even to languages in different scripts.
\citet{chi-etal-2020-finding} probe mBERT for structural phenomena and find that its representations can recover syntactic tree distances in languages other than English.
These findings present shreds of evidence that the pretrained multilingual representations do capture cross-lingual properties in various aspects.
On the flip side, a line of research shows that pretrained ML-LMs encode strong \textit{language-specific} signals.
This causes their multilingual representations to cluster by language identities instead of semantic meaning~\citep{wu-dredze-2019-beto,roy-etal-2020-lareqa,libovicky-etal-2020-language}.
The property largely hinders the expression of linguistic signals shared across languages.
For applications like cross-lingual sentence retrieval that mainly consider semantic information, ML-LMs with strong language-specific signals tend to retrieve answers from specific languages, regardless of their semantic meaning~\citep{roy-etal-2020-lareqa}.


Motivated by previous findings about language identity information, we aim to locate language-specific factors captured by the pretrained ML-LMs for recovering a \textit{language-agnostic} embedding space.
Inspired by advances in domain generalization~\citep{pmlr-v28-muandet13,NIPS2017_21c5bba1,pmlr-v119-piratla20a}, we explore a simple but effective approach, {\ours}, to discover a \textbf{L}ow-rank \textbf{S}ubspace for language-\textbf{A}gnostic \textbf{R}epresentations within an ML-LM.
The subspace primarily encodes information irrelevant to semantics, and can be identified \textit{without any translation pairs} based on singular value decomposition.
Once the subspace is found, we can directly factor out language-specific factors from the multilingual embeddings by projecting them into the null space without finetuning.

To evaluate {\ours}, we focus on semantic tasks for multilingual sentence embeddings.
On standard cross-lingual zero-shot transfer tasks including classification and sentence retrieval, {\ours} consistently achieves significant improvements.
Especially, applying {\ours} leads to significant improvements for pretrained ML-LMs on LAReQA~\citep{roy-etal-2020-lareqa}, a challenging benchmark targeting strong language agnosticism.

We further examine what information exactly the subspace contains.
By performing correlation analysis between structural language similarities obtained from the URIEL database~\citep{littell-etal-2017-uriel} and the language similarities captured on the subspace, we observe that the subspace encodes a great deal of syntactic information.
This implies that {\ours} successfully erases linguistic signals that are redundant to semantic tasks to facilitate language agnosticism.

To conclude, our main contributions are:
\begin{itemize}
    \item We present one of the pioneering efforts to discover that there exist low-rank subspaces of pretrained ML-LMs' embeddings that mainly encode language-specific signals.
    \item To identify the subspace in a ML-LM, we present a simple unsupervised approach called {\ours} based on singular value decomposition.
    By projecting embeddings onto the null space, {\ours} can exclude the unwanted factors to facilitate language agnosticism.
    \item Empirical results show that {\ours} is surprisingly effective for a variety of semantic tasks.
    We also elucidate that the subspace encodes strong syntactic signals with careful experimental analysis.
\end{itemize}




\section{Related Work}
\begin{figure*}[t]
    \centering
    \includegraphics[width=\linewidth]{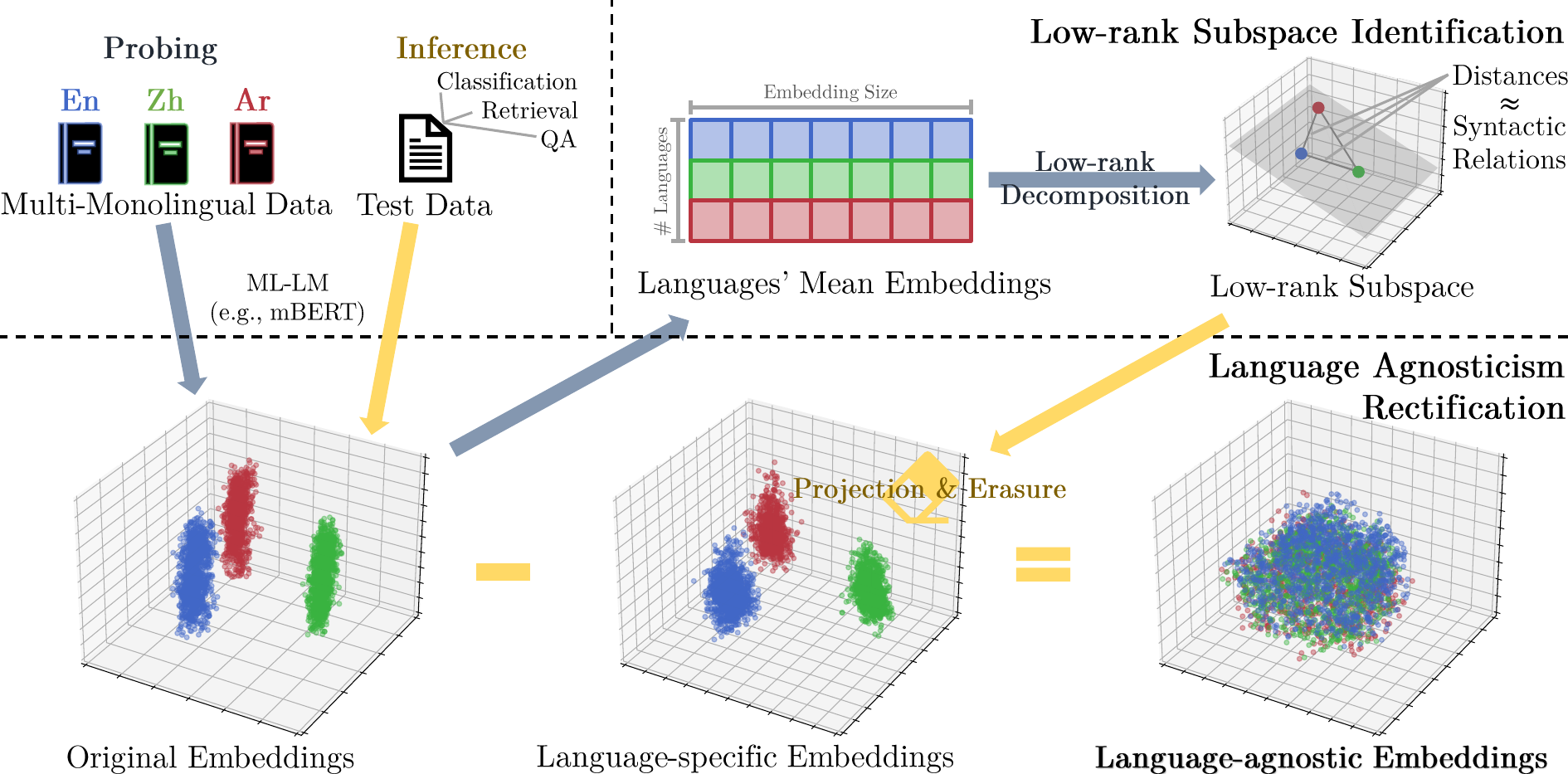}
    \caption{
    Conceptual illustration of our alignment method {\ours}.
    There exists strong language identity information from the original pretrained multilingual representations.
    By projecting away language-specific components that reside in a low-rank subspace discovered in identification process (in top-right), we can produce a language-agnostic embedding space via language agnosticism rectification (in bottom).
    The probing procedure (colored in blue-grey) and the inference procedure (colored in yellow) can be done separately.
    }
    \label{fig:illustration}
\end{figure*}


\paragraph{Understanding Pretrained Multilingual Representations}
Recently, there has been a surge of interest in probing pretrained ML-LMs like mBERT~\citep{devlin-etal-2019-bert}.
\citet{pires-etal-2019-multilingual} present an exploratory study on the cross-linguality of mBERT, showing that mBERT exhibits strong zero-shot performances for typologically similar languages.
\citet{libovicky-etal-2020-language} find that the original mBERT embeddings can be decomposed into a language-specific component and a language-neutral component.
\citet{chi-etal-2020-finding} probe mBERT for universal grammatical relations and show that mBERT does encode fine-grained syntactic distinctions across languages.
\citet{muller-etal-2021-first} find that mBERT operates as the stacking of two sub-networks and mainly the lower part of the model is crucial for cross-lingual transfer.

\paragraph{Language-agnostic Representations}
To further facilitate semantic downstream tasks like text classification, retrieval, and question answering, it is appealing to remove language-specific signals from the original embeddings without destroying the intrinsic semantic meaning.

LASER~\citep{artetxe-schwenk-2019-massively} utilizes parallel data to train a BiLSTM-based multilingual sentence encoder.
\citet{zhao-etal-2021-inducing} obtain language-agnostic embeddings from mBERT and XLM-R by explicitly aligning the word pairs and further normalizing the latent spaces with zero mean and unit variance.
\citet{yang-etal-2021-simple} regard the top principal components from each language's embedding space as the primary source of language bias and propose to project them away to boost language agnosticism.

Our work bears resemblance to~\citet{yang-etal-2021-simple}, but with clear distinctions in that: 1) we model language-specific signals \textit{jointly} in the multilingual embedding space instead of locating it \textit{separately} within each language; 2) we further verify what exactly the linguistic signals are identified, and present evidences that {\ours} primarily removes syntactic information.
A few previous works~\citep{gonen-etal-2020-greek,DBLP:journals/corr/abs-2109-08040,DBLP:journals/corr/abs-2205-10964} also attempt to locate language-agnostic embeddings in subspaces of ML-LMs.
Apart from the dissimilarity of methodology, we focus on sentence-level instead of token-level tasks and provide shreds of evidence that the identified subspace exhibits strong correlation with syntactic information.





\paragraph{Low-rank Subspaces in Other Applications}
Low-rank subspaces have been employed in various applications.
In face recognition, the most expressive features for face representations are located via subspace analysis methods like PCA~\citep{139758,1316855}.
For domain adaptation and domain generalization, a typical idea is to uncover a shared subspace on which the distribution mismatch between domains is reduced~\citep{pmlr-v28-muandet13,5640675,NIPS2017_21c5bba1}.
Recent advances in probing Generative Adversarial Networks (GANs) also observe meaningful latent subspaces that enable precise control of GAN generation~\citep{wang2021a,NEURIPS2021_8b406655}.
These findings to some extent motivate this paper.






\section{Methodology}
In this section, we first introduce our method to identify the low-rank language-specific subspace in an \textit{unsupervised} manner.
Once the subspace is found, we can then suppress the language identity from the original multilingual embeddings to achieve language agnosticism rectification by projecting them to the null space.
This post-training alignment procedure can largely benefit downstream tasks like cross-lingual retrieval which solely utilize semantic-related information.


\subsection{Multilingual Embedding Decomposition}\label{sec:preliminary}
To locate the language-specific factors,
we follow previous works~\citep{pires-etal-2019-multilingual,libovicky-etal-2020-language,yang-etal-2021-simple} to hypothesize that each multilingual embedding $\boldsymbol{e}_l \in \mathbb{R}^{d}$ in language $l$ can be decomposed in an additive form:
\begin{equation*}
    \boldsymbol{e}_l := \boldsymbol{s}_l + \boldsymbol{a}_l,
\end{equation*}
where $\boldsymbol{s}_l \in \mathbb{R}^{d}$ and $\boldsymbol{a}_l \in \mathbb{R}^{d}$ represent the language-specific component to remove and the language-agnostic component to keep, respectively.

Built on the above assumption, previous unsupervised approaches extract the language identity information \textit{separately} for each language space.
Given an ML-LM (e.g., mBERT), the extracted embeddings $\mathcal{E}_l := \{\boldsymbol{e}_l^i\}_{i=1}^{n}$ from $n$ samples of task training data or external monolingual corpora contain mixed linguistic information of semantic-relevant and semantic-irrelevant signals about language $l$.
\citet{libovicky-etal-2020-language} use the empirical mean $\frac{1}{n} \sum_{i=1}^n \boldsymbol{e}_l^i$ to obtain $\boldsymbol{s}_l$.
\citet{yang-etal-2021-simple} use the top-$k$ principal components $\boldsymbol{C}_l = \text{PCA} (\mathcal{E}_l) \in \mathbb{R}^{d \times k}$ to encode language identity signals, and propose to factor them out with $\boldsymbol{s}_l = \boldsymbol{C}_l \boldsymbol{C}_l^{\top} \boldsymbol{e}_l$ to facilitate language agnosticism.


In spite of their promising results for semantic-related tasks, these methods fall short of comprehensively discovering cross-lingual relationship in the latent space.
For each language $l$, both of them leverage solely $\mathcal{E}_l$ to locate language-specific information, which fails to distinguish itself from semantic signals as other languages' characteristics is unknown. 
Without careful tuning, this can lead to unexpected semantic information loss~\citep{khodak-etal-2018-la}.
Besides, it is also unclear what exactly language-specific signals are captured by these approaches.

\subsection{Low-rank Subspace Identification}
To alleviate the above issues, we attempt to globally capture language-specific information from the multilingual latent space.
Inspired by previous works in domain adaptation and domain generalization
\citep{pmlr-v28-muandet13,NIPS2017_21c5bba1,pmlr-v119-piratla20a}, we present a simple approach that identifies a low-rank subspace of the original multilingual latent space, $\boldsymbol{M}_s \in \mathbb{R}^{d \times r}$, spanned by $r$ components.
Intuitively, the subspace encodes language-specific signals via measuring the latent discrepancy among languages.

To be specific, we first extract the mean embedding $\boldsymbol{\mu}_l = \frac{1}{n} \sum_{i=1}^n \boldsymbol{e}_l^i$ of each language $l$ 
in the same spirit of previous approaches.
Concatenating $\boldsymbol{\mu}_l$ of $L$ languages column-by-column results in the mean embedding matrix $\boldsymbol{M} \in \mathbb{R}^{d \times L}$.
As discussed in Section~\ref{sec:preliminary}, the mean embeddings can unexpectedly mix the desired language-specific signals with semantic information.
To avoid removing the semantic information shared among languages, we decompose $\boldsymbol{M}$ into two components: 1) a vector $\boldsymbol{\mu}$ representing what is commonly shared across languages in the latent space; 2) a matrix $\boldsymbol{M}_s$ specifying a low-rank subspace on which different languages express different linguistic signals.
With the orthogonality constraint, our objective is:
\begin{equation}\label{eq:objective}
    \begin{aligned}
    \min_{\boldsymbol{\mu}, \boldsymbol{M}_{s}, \boldsymbol{\Gamma}}
    \quad& \left\|\boldsymbol{M}-\boldsymbol{\mu} \boldsymbol{\mathbbm{1}}^{\top}-\boldsymbol{M}_{s} \boldsymbol{\Gamma}^{\top}\right\|_{F}^{2}\\
    \textrm{s.t.} \quad& \boldsymbol{\mu} \perp \text{Span}\left(\boldsymbol{M}_{s}\right),
    \end{aligned}
\end{equation}
where $\boldsymbol{\Gamma} \in \mathbb{R}^{L \times r}$ is the coordinates of language-specific signals along the subspace's $r$ components and $\boldsymbol{\mathbbm{1}} \in \mathbb{R}^{d}$ contains all ones.

The optimal solution of Equation~\ref{eq:objective} can be computed efficiently via Singular Value Decomposition (SVD), as proved in Appendix~\ref{sec:proof}.
Algorithm~\ref{alg:ours} presents the detailed procedure.
The only hyperparameter $r < L$ controls the amount of language-specific information captured by the identified subspace.
The larger $r$ is, the more language-specific signals we can identify.

\subsection{Language Agnosticism Rectification}
Once we find the low-rank subspace with semantically irrelevant information encoded, we can improve language agnosticism via projecting multilingual embeddings onto the null space of $\boldsymbol{M}_s$:
\begin{equation*}
    \begin{aligned}
        \boldsymbol{a}_l
        &=\left(\boldsymbol{I}-\boldsymbol{M}_s\left(\boldsymbol{M}_s^{\top} \boldsymbol{M}_s\right)^{-1} \boldsymbol{M}_s^\top\right) \boldsymbol{e}_l \\
        &=\boldsymbol{e}_l-\boldsymbol{M}_s \boldsymbol{M}_s^\top \boldsymbol{e}_l.
    \end{aligned}
\end{equation*}
Given that usually $l \ll d$, the information removed is restricted to aspects that emerges to be language-specific and will not lead to dimensional collapse.






\SetKwComment{Comment}{/* }{ */}
\SetKwInput{KwData}{In}
\SetKwInput{KwResult}{Out}

\begin{algorithm}[t]
\caption{Language-specific Subspace Identification}\label{alg:ours}
\KwData{languages' mean embeddings $\boldsymbol{M}$, rank of subspace $r$}
\KwResult{language-agnostic component $\boldsymbol{\mu}$, language-specific subspace $\boldsymbol{M}_s$, coordinates $\boldsymbol{\Gamma}$}
\Comment{1) Approximate $\boldsymbol{M}$ in low rank}
$\boldsymbol{\mu}^\prime \gets \frac{1}{d} \boldsymbol{M} \boldsymbol{\mathbbm{1}}$\label{line:1}\;
$\boldsymbol{M}_{s}^\prime, \text{\_}, \boldsymbol{\Gamma}^\prime \gets \text{Top-} r \text{ SVD}\left(\boldsymbol{M}-\boldsymbol{\mu}^\prime \boldsymbol{\mathbbm{1}}^{\top}\right)$\;
$\boldsymbol{M}^\prime \gets \boldsymbol{\mu}^\prime \boldsymbol{\mathbbm{1}}^{\top}+\boldsymbol{M}_{s}^\prime {\boldsymbol{\Gamma}^\prime}^{\top}$\label{line:3}\;
\Comment{2) Force orthogonality}
$\boldsymbol{\mu} \gets \frac{1}{\|{\boldsymbol{M}^\prime}^{+} \boldsymbol{\mathbbm{1}}\|^2} {\boldsymbol{M}^\prime}^{+} \boldsymbol{\mathbbm{1}}$\label{line:4}\;
$\boldsymbol{M}_{s}, \text{\_}, \boldsymbol{\Gamma} \gets \text{Top-} r \text{ SVD}\left(\boldsymbol{M}^\prime-\boldsymbol{\mu} \boldsymbol{\mathbbm{1}}^{\top}\right)$\label{line:5}
\end{algorithm}



\section{Experiments}
\begin{table*}[t]
    \centering
    \begin{tabular}{ccccc}
        \thickhline
         & mBERT & XLM & XLM-R & LABSE\\
        \hline
        \multicolumn{5}{l}{\textit{Cross-lingual zero-shot transfer (w/o finetuning)}} \\
        \hline
        {\original} & 37.53{\scriptsize+00.00\%} & 28.13{\scriptsize+00.00\%} & 57.68{\scriptsize+00.00\%} & 95.47{\scriptsize+00.00\%}\\
        \demean~\citep{libovicky-etal-2020-language} & 39.57{\scriptsize+05.43\%}  & 27.13{\scriptsize-03.57\%} & 61.08{\scriptsize+05.89\%} & 95.56{\scriptsize+00.10\%}\\
        {\lir} ($k=1$) ~\citep{yang-etal-2021-simple} & 39.70{\scriptsize+05.77\%} & 28.75{\scriptsize+02.22\%} & 61.60{\scriptsize+06.80\%} & \textbf{95.63}{\scriptsize+00.16\%}\\
        {\lir} ($k=15$) ~\citep{yang-etal-2021-simple} & 41.21{\scriptsize+09.80\%} & 31.65{\scriptsize+12.51\%} & 62.80{\scriptsize+08.87\%} & 95.56{\scriptsize+00.10\%}\\
        {\ours} & \textbf{44.64}{\scriptsize+18.94\%} & \textbf{33.16}{\scriptsize+17.89\%} & \textbf{65.05}{\scriptsize+12.77\%} & 95.54{\scriptsize+00.08\%}\\
        \hline
        \multicolumn{5}{l}{\textit{Cross-lingual zero-shot transfer (w/ finetuning)}} \\
        \hline
        Full-Model-FS~\citep{xu-etal-2022-s4}$^{\dagger}$ & - & - & 60.5{\scriptsize+04.9\%}/66.2{\scriptsize+14.8\%} & - \\
        S$^4$-Tuning~\citep{xu-etal-2022-s4}$^{\dagger}$ & - & - & 66.1{\scriptsize+14.6\%}/69.5{\scriptsize+20.5\%} & - \\
        Full-Model~\citep{ruder-etal-2021-xtreme}$^{\ddagger}$ & 42.8{\scriptsize+14.0\%} & - & 76.6{\scriptsize+32.8\%} & - \\
        \thickhline
    \end{tabular}
    \caption{Retrieval accuracy (\%) on Tatoeba (averaged over all 36 languages).
    $^{\dagger}$Results from~\citet{xu-etal-2022-s4} report few-shot performances with different numbers of shots (64/128).
    $^{\ddagger}$Results are calculated from \citet{ruder-etal-2021-xtreme}.
    We use ``-'' to indicate results that are not reported in the references and use ``+\%'' to report relative improvements.
    }
    \label{tab:tatoeba}
\end{table*}


We systematically evaluate our method on various tasks followed by further analyses\footnote{Code: \url{https://github.com/fffffarmer/LSAR}.}, with the purposes of understanding: 1) whether the proposed approach can benefit downstream tasks; 2) what exactly the identified low-rank subspace captures.


To begin with, we describe our evaluation protocol for the alignment methods, which largely 
follows \citet{yang-etal-2021-simple} but with a broader scope to include more base models as listed in Section~\ref{sec:base_models}.
Given one of the pretrained ML-LMs, we first randomly collect 10,000 sentences for each language from the OSCAR corpus~\citep{ortiz-suarez-etal-2020-monolingual} covering all the evaluated languages and their web crawl texts\footnote{
\citet{yang-etal-2021-simple} use Wiki-40B~\citep{guo-etal-2020-wiki} for collecting sentence embeddings.
The corpus fails to cover all the languages evaluated in Tatoeba.
We also report the numbers using Wiki-40B as the text resource for LAReQA and Amazon Reviews in Appendix~\ref{sec:wiki40b}.}.
The sentence embeddings extracted by the pretrained model are then used for finding the low-rank subspace described in Equation~\ref{eq:objective}.
Unless otherwise indicated, we consistently report {\ours} with $r = l - 1$, where $l$ is the number of the evaluated languages.
We evaluate language agnosticism over pretrained ML-LMs that are commonly used, as described in Appendix~\ref{sec:base_models}.
Detailed results are listed in Appendix~\ref{sec:detail_results}.


\subsection{Baselines}\label{sec:baselines}
Apart from \textbf{\original} that keeps the pretrained ML-LM intact, 
we compare {\ours} with the following baselines.
The baselines share the same setting as ours in that
both of them require no parallel text and aim at removing language-specific factors in a post-training manner.

\paragraph{\demean} \citet{libovicky-etal-2020-language} extract language-neutral embeddings from the original pretrained multilingual sentence encoders via subtracting the mean embedding for each language.
The mean embeddings are calculated from the multi-monolingual OSCAR corpus.

\paragraph{\lir} \citet{yang-etal-2021-simple} propose to project away the top-$k$ principal components of each language's embeddings to facilitate language agnosticism, where $k$ is the hyperparameter.
Again, the top principal components are extracted from the multi-monolingual corpus.

\begin{figure*}[t]
     \centering
     \begin{subfigure}[b]{0.3\linewidth}
         \centering
         \includegraphics[width=\linewidth]{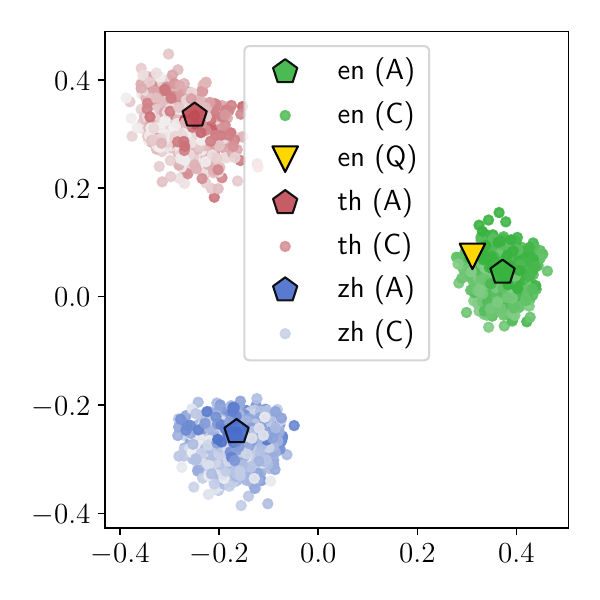}
         \caption{Original}
         \label{fig:lareqa-original}
     \end{subfigure}
     \begin{subfigure}[b]{0.3\linewidth}
         \centering
         \includegraphics[width=\linewidth]{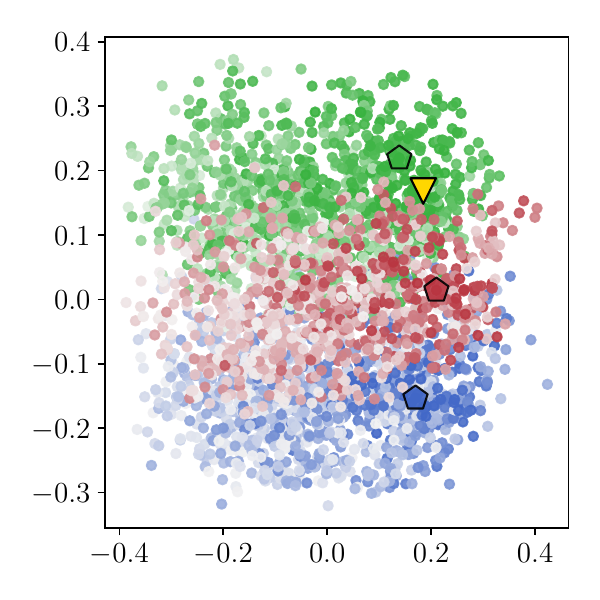}
         \caption{w/ {\lir} ($k=1$)}
         \label{fig:lareqa-lir}
     \end{subfigure}
     \begin{subfigure}[b]{0.3\linewidth}
         \centering
         \includegraphics[width=\linewidth]{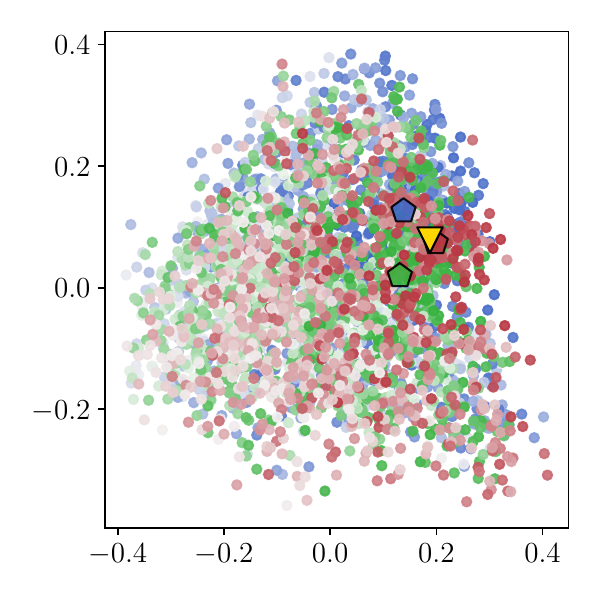}
         \caption{w/ {\ours}}
     \end{subfigure}
    \caption{2D PCA visualization on LAReQA.
    We display the embeddings collected from mBERT (X-X) on the XQuAD-R sub-dataset.
    Embeddings of the candidate answers (C) in English, Thai, and Mandarin are shown in small scatters.
    Embeddings of the question (Q) in English and the ground-truth answers (A) in English, Thai, and Mandarin are shown in large scatters.
    Higher opacity indicates higher predicted ranking (color bars: \includegraphics[height=8pt,width=20pt]{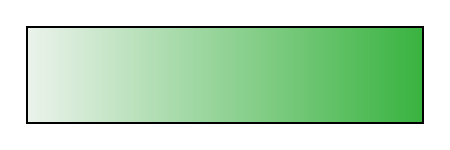}/\includegraphics[height=8pt,width=20pt]{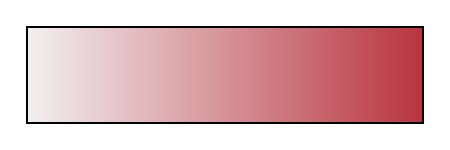}/\includegraphics[height=8pt,width=20pt]{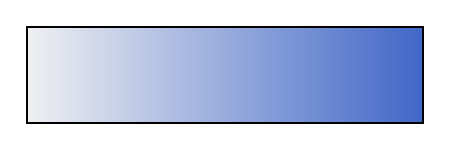}).}
    \label{fig:lareqa}
\end{figure*}

\subsection{Sentence Retrieval}\label{sec:tatoeba}
Tatoeba~\citep{artetxe-schwenk-2019-massively} is a commonly used dataset for evaluating ML-LMs.
It comprises up to 1,000 sentences for each language along with their English translations.
We follow the evaluation procedure of XTREME~\citep{pmlr-v119-hu20b} that covers 36 languages.
For each language pair, we go through each sentence in the source language and find the closest sentence in the target language using cosine similarity.

The top-1 retrieval accuracy results are shown in Table~\ref{tab:tatoeba}.
For mBERT~\citep{devlin-etal-2019-bert}, XLM~\citep{NEURIPS2019_c04c19c2}, and XLM-R~\citep{conneau-etal-2020-unsupervised}, applying {\ours} brings significant performance gains of up to 19\% relative improvement.
Compared with {\demean} and {\lir} which separately remove information for each language, {\ours} jointly utilizes the encoded information from all the languages to better locate language-specific factors.
Furthermore, we observe that {\ours} consistently achieves the best results with hyperparameter $r$ (the rank of the low-rank subspace) equal to the number of the evaluated languages, as shown in Appendix~\ref{sec:hyperparameter}.
As the languages are diversely distributed, it is reasonable that each language possesses its own linguistic characteristics, resulting in a larger language-specific subspace to factor out.
In contrast, we find that {\lir} is vulnerable to its hyperparameter $k$ (the number of the removed principal components), which is best shown in Figure~\ref{fig:ranks}.

For LABSE~\citep{feng-etal-2022-language}, all the methods fail to provide marked enhancement.
This can be mainly attributed to the fact that LABSE already uses parallel corpora to explicitly align multilingual embeddings.
Despite that the improvement is marginal, it is still promising to combine {\ours} with existing pretraining objectives to produce better language-agnostic embeddings.



We also include several representative baselines that finetune either mBERT or XLM-R for better cross-lingual transfer results.
Although these methods are not directly comparable to ours, we believe it provides additional valuable findings to include them.
Full-Model-FS and S$^4$-Tuning finetune XLM-R on full English labeled examples and then $K$-shot data over target languages ($K=64/128$).
For Full-Model, the pretrained models are finetuned on the English SQuAD data.
On mBERT, {\ours} outperforms Full-Model by a large margin.
We also observe on XLM-R that {\ours} is competitive with finetuning the full model on 128-shot data as well as finetuning a dedicated language sub-network (S$^4$-Tuning) on 64-shot data.
The results are quite promising given that we obtain better performances with the original encoders intact and no task-relevant training data.




\begin{table}[t]
    \centering
    \begin{tabular}{c|cc|cc}
        \thickhline
        & \multicolumn{2}{c|}{XQuAD-R} & \multicolumn{2}{c}{MLQA-R}\\
        & En-En & X-X & En-En & X-X \\
        \hline
        \original & 28.57 & 23.36 & 35.71 & 26.21\\
        \demean & 35.37 & 44.66 & 35.36 & 42.14\\
        {\lir} ($k=1$) & 37.70 & 44.25 & 38.03 & 41.96\\
        \ours & \textbf{41.13} & \textbf{45.89} & \textbf{40.55} & \textbf{43.32}\\
        \thickhline
    \end{tabular}
    \caption{Answer retrieval mAP (\%) on XQuAD-R and MLQA-R of LAReQA (averaged over all languages).}
    \label{tab:lareqa}
\end{table}

\subsection{Language-agonstic Answer Retrieval}
While Tatoeba reveals the cross-lingual transferability across English-centric language pairs, it is restricted to
monolingual pools (i.e., the set of candidates is restricted to certain language).
Therefore, it fails to thoroughly evaluate whether texts with a similar semantic meaning are grouped together in the latent space, regardless of their languages.

With that in mind, we further examine the alignment methods on LAReQA~\citep{roy-etal-2020-lareqa}, a challenging cross-lingual answer retrieval task.
Unlike Tatoeba, the targets of LAReQA must be retrieved from a large multilingual candidate pool.
It consists of two sub-datasets, XQuAD-R and MLQA-R, whose candidate pool covers 11 and 7 languages respectively.

We follow~\citet{yang-etal-2021-simple} to evaluate the alignment methods on two models, \textbf{mBERT (En-En)} and \textbf{mBERT (X-X)}.
Specifically, mBERT (En-En) finetunes the original mBERT model on the English QA pairs collected from the SQuAD v1.1 dataset.
mBERT (X-X) employs the same training procedure but with an extended dataset where each sample is translated into the 11 XQuAD languages.
Since all positive samples for finetuning are within the same language as the question query, both models exhibit strong self-language bias while preserving the weak alignment property.
For evaluation, we use the dot product of embeddings to score a QA pair, which accords with the finetuning protocol.
The retrieval performance is measured by mean Average Precision (mAP).


Table~\ref{tab:lareqa} reports our LAReQA results.
We can observe that applying {\ours} again results in signification improvements, nearly doubling mAP of mBERT (X-X) on XQuAD-R.
Since in the candidate pool each language has one of the relevant answers, better retrieval performances directly indicate better language agnosticism.
{\demean} and {\lir} ($k=1$) also show impressive performances, suggesting that in weakly aligned multilingual systems, the mean embeddings and principal components do encode language-specific signals.
But for {\lir}, it is shown that removing the first principal component consistently leads to the best performance.
This is opposite to what we observe on Tatoeba, where the optimal $k$ is around 15.

To further illustrate the degree of language agnosticism, we project an English question (\textit{What theory best explains gravity?}) as well as all candidates and the ground-truth answers in English, Thai, and Mandarin via PCA.
As plotted in Figure~\ref{fig:lareqa}, candidates in English are retrieved from mBERT (X-X) with higher priority than those in Thai and Mandarin.
Applying {\ours} can effectively eliminate strong language identity information residing in the original embedding space and draw closer the question and answers from different languages.
{\lir} with $k=1$, however, falls short of rectifying language-specific signals as illustrated by the embedding spectrum in Figure~\ref{fig:lareqa-lir}.

\begin{table}[t]
    \centering
    \begin{tabular}{c|ccc}
        \thickhline
         & mBERT & XLM & XLM-R \\
        \hline
        {\original} & 74.73 & 75.31 & 80.32 \\
        {\lir} ($k=1$) & 75.39 & 75.73 & 81.14 \\
        {\ours} ($r=1$) & \textbf{75.58} & 74.93 & 81.47 \\
        {\ours} ($r=2$) & 75.49 & \textbf{75.85} & \textbf{82.37} \\
        {\ours} & 75.24 & 75.27 & 81.25 \\
        \thickhline
    \end{tabular}
    \caption{
    Classification accuracy (\%) on Amazon Reviews (averaged over English, French, German and Japanese).
    We exclude {\demean} as the embeddings are already normalized and hence {\demean} produces the same results as {\original}.
    The results of LABSE are placed in Appendix~\ref{sec:appendix} due to limited space.
    }
    \label{tab:amazon}
\end{table}

\begin{table}[b]
    \centering
    \begin{tabular}{c|ccc}
        \thickhline
        & mBERT & XLM & XLM-R \\
        \hline
        {\original} & 0.2815 & 0.5422 & 0.2457 \\
        {\demean} & 0.0975 & 0.2483 & 0.2004 \\
        {\lir} ($k=1$) & 0.0900 & 0.1875 & 0.2203 \\
        {\ours} & \textbf{0.0801} & \textbf{0.1320} & \textbf{0.0856} \\
        \thickhline
    \end{tabular}
    \caption{
    Clustering performance (NMI) of embeddings obtained by mBERT on Tatoeba.
    The results of LABSE are placed in Appendix~\ref{sec:appendix} due to limited space.
    }
    \label{tab:cluster}
\end{table}

\begin{figure}[t]
    \centering

    \begin{subfigure}[b]{0.5\linewidth}
         \centering
         \includegraphics[width=\linewidth]{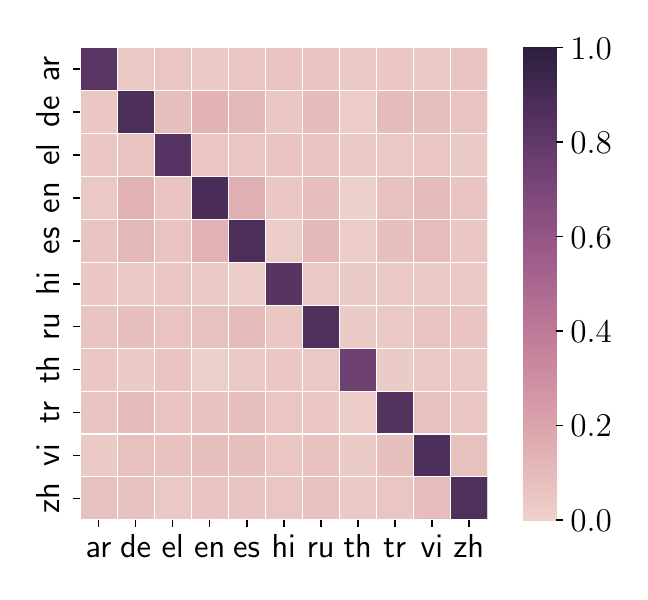}
         \caption{\original}
     \end{subfigure}
     \hspace{-10pt}
     \begin{subfigure}[b]{0.5\linewidth}
         \centering
         \includegraphics[width=\linewidth]{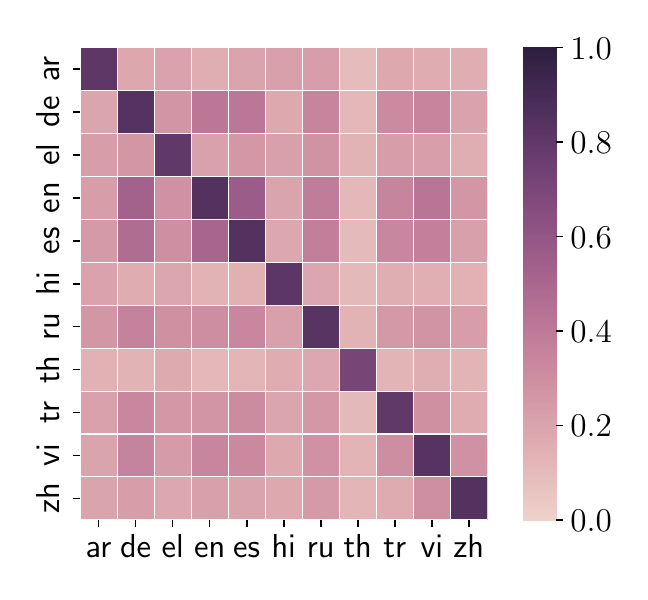}
         \caption{\ours}
     \end{subfigure}
    \caption{
    Answer retrieval mAP on XQuAD-R broken down by question language (row) and answer language (column), with model mBERT (X-X).
    Only one correct answer is included in the multilingual candidate pool.}
    \label{fig:limit_to_one}
\end{figure}

\subsection{Zero-shot Classification}

We also include the Amazon Reviews classification task~\citep{prettenhofer-stein-2010-cross} to assess zero-shot cross-lingual transfer.
The dataset consists of product reviews in English, French, German, and Japanese.
Each review is labeled as positive or negative, making it a binary classification task.
We use the same procedure to extract sentence embeddings as in Section~\ref{sec:tatoeba}, and normalize them to make regularization hyperparameters more consistent across languages.
Appendix~\ref{sec:hyperparameter} specifies how we select hyperparameters.
Following \citep{yang-etal-2021-simple}, the performance is evaluated via training a logistic regression classifier\footnote{\href{https://scikit-learn.org/stable/modules/generated/sklearn.linear_model.LogisticRegressionCV.html}{sklearn.linear\_model.LogisticRegressionCV()}.} on the English training data and then evaluating it on the test sets of all four languages.


From Table~\ref{tab:amazon}, we observe that the classifier trained on English data benefits from {\ours} for classifying reviews based on semantics as the language-specific factors are effectively erased.
Another interesting observation is that unlike sentence retrieval, removing more directions does not result in better performance.
This indicates that classification tasks can be more sensitive to semantic information.




\subsection{Analysis}

In this section, we present analysis on a variety of aspects towards what exactly language-specific information {\ours} captures.

\subsubsection{Language-specific Signals are Rectified}
From previous findings, we conjecture that our method achieves impressive cross-lingual performance by effectively removing language identity signals.
To quantitatively verify this, we measure the strength of language identity information from the perspective of clustering quality.
If the embeddings are clustered by language types, we can generally state that language-specific signals still play a prominent role in the multilingual latent space.

We perform K-Means clustering on sentence representations of Tatoeba with the number of clusters equal to the number of languages, and then evaluate the resulting clusters using the Normalized Mutual Information (NMI) metric~\citep{jawahar-etal-2019-bert}\footnote{\href{https://scikit-learn.org/stable/modules/generated/sklearn.metrics.normalized_mutual_info_score.html}{sklearn.metrics.normalized\_mutual\_info\_score()}.}.
As shown in Table~\ref{tab:cluster}, the original pretrained embeddings have relatively high NMI scores, suggesting the existence of strong language identity information. 
Our method consistently achieves smaller NMI scores.
This indicates that the embeddings have a lower tendency to group by language types since {\ours} successfully winnows down language-specific information.

The same conclusion can be drawn from the limit-to-one-target setting of LAReQA~\cite{roy-etal-2020-lareqa}.
Specifically, we remove 10 targets from the multilingual pool of XQuAD-R to evaluate on each target separately.
We choose the most biased X-X variant as the base model.
The heatmaps in Figure~\ref{fig:limit_to_one} show for each question language (row), the retrieval mAP on the pool containing just one target in different answer languages (column).
Since X-X has strong self-language bias, {\original} shows better performance on the diagonal than off-diagonal.
After applying {\ours}, we observe a significant increase in average off-diagonal performance (23.76\% vs. 5.89\%), without sacrificing much on-diagonal  performance (81.57\% vs. 84.57\%).
This again verifies that applying {\ours} effectively removes language-specific information.

\begin{figure}[t]
    \centering
    \begin{subfigure}[b]{\linewidth}
         \centering
         \includegraphics[width=\linewidth]{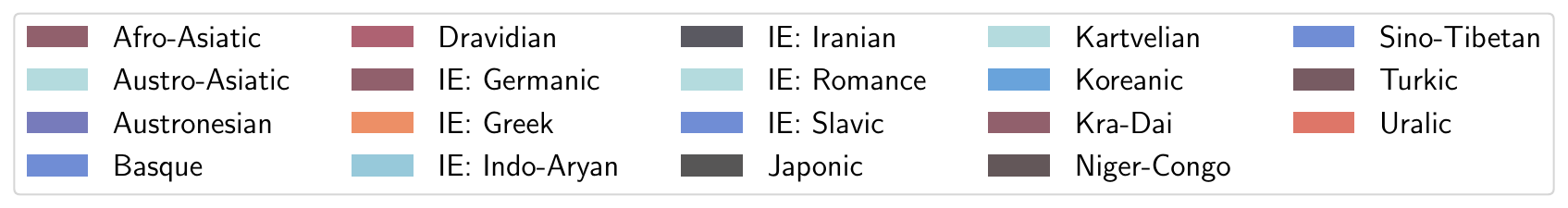}
     \end{subfigure}
     
    \begin{subfigure}[b]{0.5\linewidth}
         \centering
         \includegraphics[width=\linewidth]{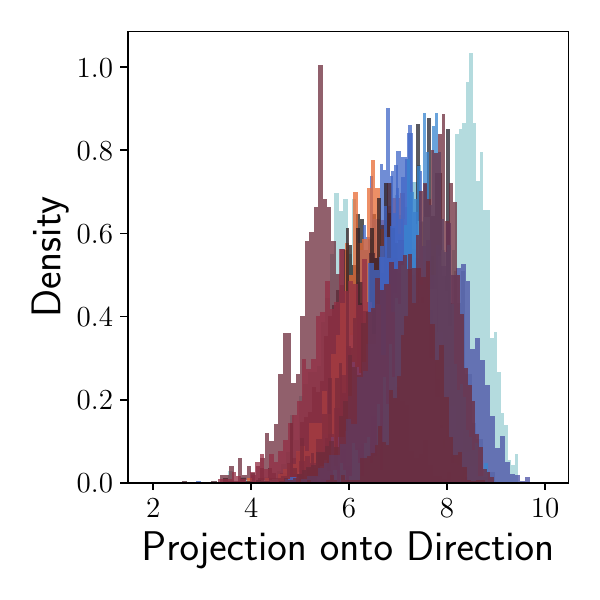}
         \caption{$1^{\text{st}}$ direction}
     \end{subfigure}
     \hspace{-10pt}
     \begin{subfigure}[b]{0.5\linewidth}
         \centering
         \includegraphics[width=\linewidth]{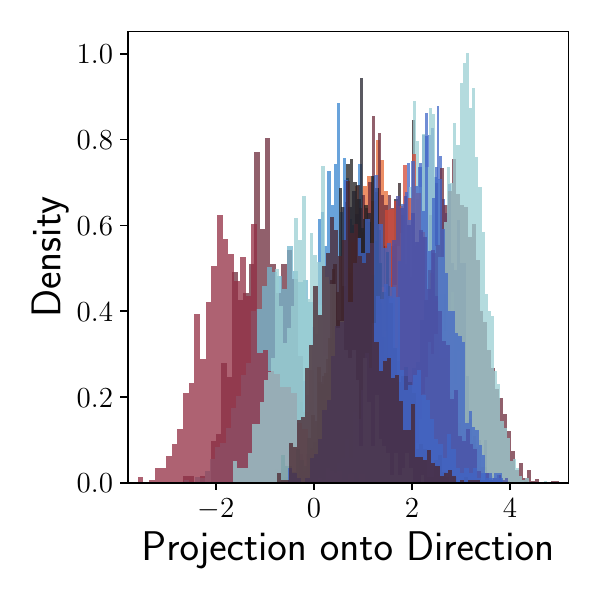}
         \caption{$2^{\text{nd}}$ direction}
     \end{subfigure}
    \caption{Removed components along the top two basis vectors of the identified low-rank subspace on mBERT.}
    \label{fig:direction}
\end{figure}

\begin{figure}[t]
    \centering
    \includegraphics[width=\linewidth]{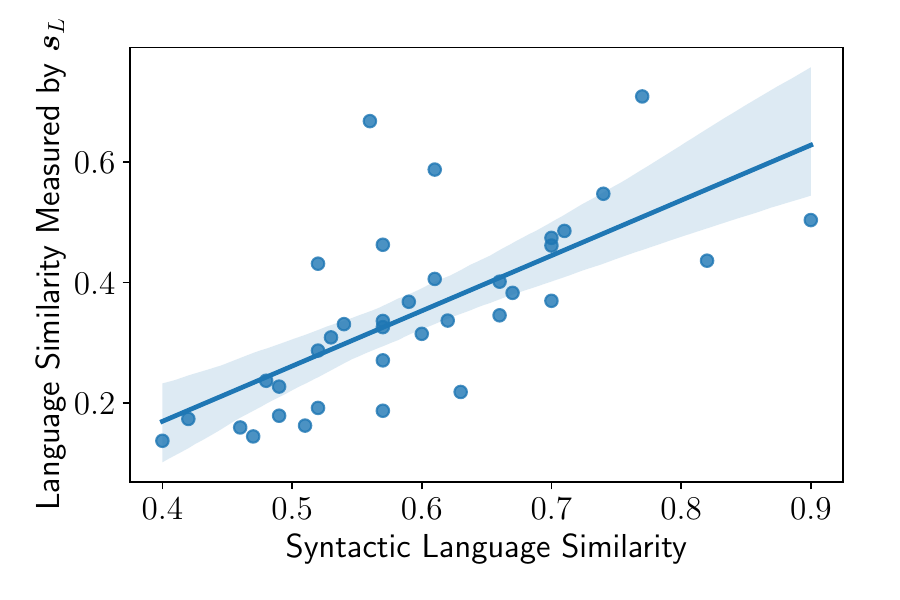}
    \caption{Language similarity obtained from syntactic signals vs. language similarity measured by language-specific $\boldsymbol{s}_L$ of mBERT.
    Each point is a language.}
\end{figure}

\subsubsection{Removed Components Form Groups of Language Families}
We next examine whether the removed components found by the low-rank subspace are truly language-specific.
This is demonstrated via plotting the removed components for different languages along top basis vectors of the subspace.
For the ease of visualization, we group them by language family.

Figure~\ref{fig:direction} shows the histograms of removed components along the top two basis vectors extracted from mBERT on 36 languages of Tatoeba, according to Equation~\ref{eq:objective}.
We can observe that the removed components disperse in groups of language families along these directions.
This implies that the identified subspace do capture language-specific signals and hence removing them along the basis vectors can narrow down latent discrepancy.

\subsubsection{The Identified Subspace Primarily Encodes Syntactic Information}
Finally, given that the removed components are language-specific, we investigate to what extent the low-rank subspace encodes typological relations among languages.
Specifically, we use the URIEL database~\citep{littell-etal-2017-uriel} to collect distances between English and other languages set out by experts based on certain typological information (e.g., syntax and phonology).
We then compare the typological distances with languages similarities obtained from the removed language-specific embeddings $\boldsymbol{s}_L$ as well as the resulting language-agnostic embeddings
$\boldsymbol{a}_L$ by calculating the cosine similarity between languages' mean embeddings.

Among all types of typological signals listed in URIEL, we find that the removed language-specific factors are mostly correlated with syntactic information.
Table~\ref{tab:correlation} shows the Pearson correlations on English and other 36 languages from Tatoeba.
The removed language-specific component $\boldsymbol{s}_L$ is highly correlated with syntactic information, whereas the correlation is much smaller in the language-agnostic embedding space with $\boldsymbol{s}_L$ removed.
This finding is in line with previous works~\citep{chi-etal-2020-finding,zhao-etal-2021-inducing} that observe the pretrained multilingual models encode rich syntactic information.

We find no prominent correlation between the removed components along certain basis vectors of the subspace and typological information.
As we do not presuppose any correspondence between basis vectors and linguistic signals, a specific basis vector falls short of individually encoding language-specific information.

\begin{table}[t]
    \centering
    \begin{tabular}{c|cccc}
        \thickhline
        & mBERT & XLM & XLM-R & LABSE \\
        \hline
        $\boldsymbol{s}_L$ & 0.6910 & 0.6378 & 0.7526 & 0.6894 \\
        $\boldsymbol{a}_L$ & -0.2711 & 0.2239 & 0.1338 & -0.2362 \\
        \thickhline
    \end{tabular}
    \caption{Pearson correlations between syntactic language similarities obtained from the URIEL database, and the language similarities obtained from language-specific $\boldsymbol{s}_L$ as well as language-agnostic $\boldsymbol{a}_L$.}
    \label{tab:correlation}
\end{table}




\section{Conclusion}
We present a simple yet effective approach called {\ours} to boost language agnosticism for pretrained multilingual encoders.
{\ours} identifies a low-rank subspace residing in a pretrained model that primarily encodes language-specific signals in an unsupervised manner via singular value decomposition.
Once the subspace is discovered, it can be used to efficiently project away the language identity information.
Empirical results demonstrate the great effectiveness of {\ours} on semantic tasks and shed light on its ability to locate syntactic relations between languages.

\section*{Limitations}
Our method {\ours} is designed and evaluated for semantic tasks.
For future work, we are interested in continuing our study for locating more fine-grained linguistic information, which can potentially boost a larger variety of downstream tasks.
While the simplicity of the proposed {\ours} is appealing, it also opens up directions for future work by generalizing the first-moment mean embeddings to higher-moment statistics and combining with pretraining objectives in more sophisticated ways.


\bibliography{emnlp2022}

\begin{thebibliography}{44}
\expandafter\ifx\csname natexlab\endcsname\relax\def\natexlab#1{#1}\fi

\bibitem[{Artetxe and Schwenk(2019)}]{artetxe-schwenk-2019-massively}
Mikel Artetxe and Holger Schwenk. 2019.
\newblock \href {https://doi.org/10.1162/tacl_a_00288} {Massively multilingual
  sentence embeddings for zero-shot cross-lingual transfer and beyond}.
\newblock \emph{Transactions of the Association for Computational Linguistics},
  7:597--610.

\bibitem[{Buitinck et~al.(2013)Buitinck, Louppe, Blondel, Pedregosa, Mueller,
  Grisel, Niculae, Prettenhofer, Gramfort, Grobler, Layton, VanderPlas, Joly,
  Holt, and Varoquaux}]{sklearn_api}
Lars Buitinck, Gilles Louppe, Mathieu Blondel, Fabian Pedregosa, Andreas
  Mueller, Olivier Grisel, Vlad Niculae, Peter Prettenhofer, Alexandre
  Gramfort, Jaques Grobler, Robert Layton, Jake VanderPlas, Arnaud Joly, Brian
  Holt, and Ga{\"{e}}l Varoquaux. 2013.
\newblock {API} design for machine learning software: experiences from the
  scikit-learn project.
\newblock In \emph{ECML PKDD Workshop: Languages for Data Mining and Machine
  Learning}, pages 108--122.

\bibitem[{Chang et~al.(2022)Chang, Tu, and
  Bergen}]{DBLP:journals/corr/abs-2205-10964}
Tyler~A. Chang, Zhuowen Tu, and Benjamin~K. Bergen. 2022.
\newblock \href {https://doi.org/10.48550/arXiv.2205.10964} {The geometry of
  multilingual language model representations}.
\newblock \emph{CoRR}, abs/2205.10964.

\bibitem[{Chi et~al.(2020)Chi, Hewitt, and Manning}]{chi-etal-2020-finding}
Ethan~A. Chi, John Hewitt, and Christopher~D. Manning. 2020.
\newblock \href {https://doi.org/10.18653/v1/2020.acl-main.493} {Finding
  universal grammatical relations in multilingual {BERT}}.
\newblock In \emph{Proceedings of the 58th Annual Meeting of the Association
  for Computational Linguistics}, pages 5564--5577, Online. Association for
  Computational Linguistics.

\bibitem[{Conneau et~al.(2020{\natexlab{a}})Conneau, Khandelwal, Goyal,
  Chaudhary, Wenzek, Guzm{\'a}n, Grave, Ott, Zettlemoyer, and
  Stoyanov}]{conneau-etal-2020-unsupervised}
Alexis Conneau, Kartikay Khandelwal, Naman Goyal, Vishrav Chaudhary, Guillaume
  Wenzek, Francisco Guzm{\'a}n, Edouard Grave, Myle Ott, Luke Zettlemoyer, and
  Veselin Stoyanov. 2020{\natexlab{a}}.
\newblock \href {https://doi.org/10.18653/v1/2020.acl-main.747} {Unsupervised
  cross-lingual representation learning at scale}.
\newblock In \emph{Proceedings of the 58th Annual Meeting of the Association
  for Computational Linguistics}, pages 8440--8451, Online. Association for
  Computational Linguistics.

\bibitem[{Conneau and Lample(2019)}]{NEURIPS2019_c04c19c2}
Alexis Conneau and Guillaume Lample. 2019.
\newblock \href
  {https://proceedings.neurips.cc/paper/2019/file/c04c19c2c2474dbf5f7ac4372c5b9af1-Paper.pdf}
  {Cross-lingual language model pretraining}.
\newblock In \emph{Advances in Neural Information Processing Systems},
  volume~32. Curran Associates, Inc.

\bibitem[{Conneau et~al.(2018)Conneau, Rinott, Lample, Williams, Bowman,
  Schwenk, and Stoyanov}]{conneau-etal-2018-xnli}
Alexis Conneau, Ruty Rinott, Guillaume Lample, Adina Williams, Samuel Bowman,
  Holger Schwenk, and Veselin Stoyanov. 2018.
\newblock \href {https://doi.org/10.18653/v1/D18-1269} {{XNLI}: Evaluating
  cross-lingual sentence representations}.
\newblock In \emph{Proceedings of the 2018 Conference on Empirical Methods in
  Natural Language Processing}, pages 2475--2485, Brussels, Belgium.
  Association for Computational Linguistics.

\bibitem[{Conneau et~al.(2020{\natexlab{b}})Conneau, Wu, Li, Zettlemoyer, and
  Stoyanov}]{conneau-etal-2020-emerging}
Alexis Conneau, Shijie Wu, Haoran Li, Luke Zettlemoyer, and Veselin Stoyanov.
  2020{\natexlab{b}}.
\newblock \href {https://doi.org/10.18653/v1/2020.acl-main.536} {Emerging
  cross-lingual structure in pretrained language models}.
\newblock In \emph{Proceedings of the 58th Annual Meeting of the Association
  for Computational Linguistics}, pages 6022--6034, Online. Association for
  Computational Linguistics.

\bibitem[{Devlin et~al.(2019)Devlin, Chang, Lee, and
  Toutanova}]{devlin-etal-2019-bert}
Jacob Devlin, Ming-Wei Chang, Kenton Lee, and Kristina Toutanova. 2019.
\newblock \href {https://doi.org/10.18653/v1/N19-1423} {{BERT}: Pre-training of
  deep bidirectional transformers for language understanding}.
\newblock In \emph{Proceedings of the 2019 Conference of the North {A}merican
  Chapter of the Association for Computational Linguistics: Human Language
  Technologies, Volume 1 (Long and Short Papers)}, pages 4171--4186,
  Minneapolis, Minnesota. Association for Computational Linguistics.

\bibitem[{Eckart and Young(1936)}]{eckart1936approximation}
Carl Eckart and Gale Young. 1936.
\newblock The approximation of one matrix by another of lower rank.
\newblock \emph{Psychometrika}, 1(3):211--218.

\bibitem[{Feng et~al.(2022)Feng, Yang, Cer, Arivazhagan, and
  Wang}]{feng-etal-2022-language}
Fangxiaoyu Feng, Yinfei Yang, Daniel Cer, Naveen Arivazhagan, and Wei Wang.
  2022.
\newblock \href {https://aclanthology.org/2022.acl-long.62} {Language-agnostic
  {BERT} sentence embedding}.
\newblock In \emph{Proceedings of the 60th Annual Meeting of the Association
  for Computational Linguistics (Volume 1: Long Papers)}, pages 878--891,
  Dublin, Ireland. Association for Computational Linguistics.

\bibitem[{Gonen et~al.(2020)Gonen, Ravfogel, Elazar, and
  Goldberg}]{gonen-etal-2020-greek}
Hila Gonen, Shauli Ravfogel, Yanai Elazar, and Yoav Goldberg. 2020.
\newblock \href {https://doi.org/10.18653/v1/2020.blackboxnlp-1.5} {It{'}s not
  {G}reek to m{BERT}: Inducing word-level translations from multilingual
  {BERT}}.
\newblock In \emph{Proceedings of the Third BlackboxNLP Workshop on Analyzing
  and Interpreting Neural Networks for NLP}, pages 45--56, Online. Association
  for Computational Linguistics.

\bibitem[{Guo et~al.(2020)Guo, Dai, Vrande{\v{c}}i{\'c}, and
  Al-Rfou}]{guo-etal-2020-wiki}
Mandy Guo, Zihang Dai, Denny Vrande{\v{c}}i{\'c}, and Rami Al-Rfou. 2020.
\newblock \href {https://aclanthology.org/2020.lrec-1.297} {{W}iki-40{B}:
  Multilingual language model dataset}.
\newblock In \emph{Proceedings of the 12th Language Resources and Evaluation
  Conference}, pages 2440--2452, Marseille, France. European Language Resources
  Association.

\bibitem[{Hu et~al.(2020)Hu, Ruder, Siddhant, Neubig, Firat, and
  Johnson}]{pmlr-v119-hu20b}
Junjie Hu, Sebastian Ruder, Aditya Siddhant, Graham Neubig, Orhan Firat, and
  Melvin Johnson. 2020.
\newblock \href {https://proceedings.mlr.press/v119/hu20b.html} {{XTREME}: A
  massively multilingual multi-task benchmark for evaluating cross-lingual
  generalisation}.
\newblock In \emph{Proceedings of the 37th International Conference on Machine
  Learning}, volume 119 of \emph{Proceedings of Machine Learning Research},
  pages 4411--4421. PMLR.

\bibitem[{Jawahar et~al.(2019)Jawahar, Sagot, and
  Seddah}]{jawahar-etal-2019-bert}
Ganesh Jawahar, Beno{\^\i}t Sagot, and Djam{\'e} Seddah. 2019.
\newblock \href {https://doi.org/10.18653/v1/P19-1356} {What does {BERT} learn
  about the structure of language?}
\newblock In \emph{Proceedings of the 57th Annual Meeting of the Association
  for Computational Linguistics}, pages 3651--3657, Florence, Italy.
  Association for Computational Linguistics.

\bibitem[{K et~al.(2020)K, Wang, Mayhew, and Roth}]{K2020Cross-Lingual}
Karthikeyan K, Zihan Wang, Stephen Mayhew, and Dan Roth. 2020.
\newblock \href {https://openreview.net/forum?id=HJeT3yrtDr} {Cross-lingual
  ability of multilingual bert: An empirical study}.
\newblock In \emph{International Conference on Learning Representations}.

\bibitem[{Khodak et~al.(2018)Khodak, Saunshi, Liang, Ma, Stewart, and
  Arora}]{khodak-etal-2018-la}
Mikhail Khodak, Nikunj Saunshi, Yingyu Liang, Tengyu Ma, Brandon Stewart, and
  Sanjeev Arora. 2018.
\newblock \href {https://doi.org/10.18653/v1/P18-1002} {A la carte embedding:
  Cheap but effective induction of semantic feature vectors}.
\newblock In \emph{Proceedings of the 56th Annual Meeting of the Association
  for Computational Linguistics (Volume 1: Long Papers)}, pages 12--22,
  Melbourne, Australia. Association for Computational Linguistics.

\bibitem[{Lewis et~al.(2020)Lewis, Oguz, Rinott, Riedel, and
  Schwenk}]{lewis-etal-2020-mlqa}
Patrick Lewis, Barlas Oguz, Ruty Rinott, Sebastian Riedel, and Holger Schwenk.
  2020.
\newblock \href {https://doi.org/10.18653/v1/2020.acl-main.653} {{MLQA}:
  Evaluating cross-lingual extractive question answering}.
\newblock In \emph{Proceedings of the 58th Annual Meeting of the Association
  for Computational Linguistics}, pages 7315--7330, Online. Association for
  Computational Linguistics.

\bibitem[{Liang et~al.(2021)Liang, Dufter, and
  Sch{\"{u}}tze}]{DBLP:journals/corr/abs-2109-08040}
Sheng Liang, Philipp Dufter, and Hinrich Sch{\"{u}}tze. 2021.
\newblock \href {http://arxiv.org/abs/2109.08040} {Locating language-specific
  information in contextualized embeddings}.
\newblock \emph{CoRR}, abs/2109.08040.

\bibitem[{Libovick{\'y} et~al.(2020)Libovick{\'y}, Rosa, and
  Fraser}]{libovicky-etal-2020-language}
Jind{\v{r}}ich Libovick{\'y}, Rudolf Rosa, and Alexander Fraser. 2020.
\newblock \href {https://doi.org/10.18653/v1/2020.findings-emnlp.150} {On the
  language neutrality of pre-trained multilingual representations}.
\newblock In \emph{Findings of the Association for Computational Linguistics:
  EMNLP 2020}, pages 1663--1674, Online. Association for Computational
  Linguistics.

\bibitem[{Littell et~al.(2017)Littell, Mortensen, Lin, Kairis, Turner, and
  Levin}]{littell-etal-2017-uriel}
Patrick Littell, David~R. Mortensen, Ke~Lin, Katherine Kairis, Carlisle Turner,
  and Lori Levin. 2017.
\newblock \href {https://aclanthology.org/E17-2002} {{URIEL} and lang2vec:
  Representing languages as typological, geographical, and phylogenetic
  vectors}.
\newblock In \emph{Proceedings of the 15th Conference of the {E}uropean Chapter
  of the Association for Computational Linguistics: Volume 2, Short Papers},
  pages 8--14, Valencia, Spain. Association for Computational Linguistics.

\bibitem[{Liu et~al.(2019)Liu, Ott, Goyal, Du, Joshi, Chen, Levy, Lewis,
  Zettlemoyer, and Stoyanov}]{liu2019roberta}
Yinhan Liu, Myle Ott, Naman Goyal, Jingfei Du, Mandar Joshi, Danqi Chen, Omer
  Levy, Mike Lewis, Luke Zettlemoyer, and Veselin Stoyanov. 2019.
\newblock Roberta: A robustly optimized bert pretraining approach.
\newblock \emph{arXiv preprint arXiv:1907.11692}.

\bibitem[{Motiian et~al.(2017)Motiian, Jones, Iranmanesh, and
  Doretto}]{NIPS2017_21c5bba1}
Saeid Motiian, Quinn Jones, Seyed Iranmanesh, and Gianfranco Doretto. 2017.
\newblock \href
  {https://proceedings.neurips.cc/paper/2017/file/21c5bba1dd6aed9ab48c2b34c1a0adde-Paper.pdf}
  {Few-shot adversarial domain adaptation}.
\newblock In \emph{Advances in Neural Information Processing Systems},
  volume~30. Curran Associates, Inc.

\bibitem[{Muandet et~al.(2013)Muandet, Balduzzi, and
  Schölkopf}]{pmlr-v28-muandet13}
Krikamol Muandet, David Balduzzi, and Bernhard Schölkopf. 2013.
\newblock \href {https://proceedings.mlr.press/v28/muandet13.html} {Domain
  generalization via invariant feature representation}.
\newblock In \emph{Proceedings of the 30th International Conference on Machine
  Learning}, volume~28 of \emph{Proceedings of Machine Learning Research},
  pages 10--18, Atlanta, Georgia, USA. PMLR.

\bibitem[{Muller et~al.(2021)Muller, Elazar, Sagot, and
  Seddah}]{muller-etal-2021-first}
Benjamin Muller, Yanai Elazar, Beno{\^\i}t Sagot, and Djam{\'e} Seddah. 2021.
\newblock \href {https://doi.org/10.18653/v1/2021.eacl-main.189} {First align,
  then predict: Understanding the cross-lingual ability of multilingual
  {BERT}}.
\newblock In \emph{Proceedings of the 16th Conference of the European Chapter
  of the Association for Computational Linguistics: Main Volume}, pages
  2214--2231, Online. Association for Computational Linguistics.

\bibitem[{Ortiz~Su{\'a}rez et~al.(2020)Ortiz~Su{\'a}rez, Romary, and
  Sagot}]{ortiz-suarez-etal-2020-monolingual}
Pedro~Javier Ortiz~Su{\'a}rez, Laurent Romary, and Beno{\^\i}t Sagot. 2020.
\newblock \href {https://doi.org/10.18653/v1/2020.acl-main.156} {A monolingual
  approach to contextualized word embeddings for mid-resource languages}.
\newblock In \emph{Proceedings of the 58th Annual Meeting of the Association
  for Computational Linguistics}, pages 1703--1714, Online. Association for
  Computational Linguistics.

\bibitem[{Pan et~al.(2011)Pan, Tsang, Kwok, and Yang}]{5640675}
Sinno~Jialin Pan, Ivor~W. Tsang, James~T. Kwok, and Qiang Yang. 2011.
\newblock \href {https://doi.org/10.1109/TNN.2010.2091281} {Domain adaptation
  via transfer component analysis}.
\newblock \emph{IEEE Transactions on Neural Networks}, 22(2):199--210.

\bibitem[{Peters et~al.(2018)Peters, Neumann, Iyyer, Gardner, Clark, Lee, and
  Zettlemoyer}]{peters-etal-2018-deep}
Matthew~E. Peters, Mark Neumann, Mohit Iyyer, Matt Gardner, Christopher Clark,
  Kenton Lee, and Luke Zettlemoyer. 2018.
\newblock \href {https://doi.org/10.18653/v1/N18-1202} {Deep contextualized
  word representations}.
\newblock In \emph{Proceedings of the 2018 Conference of the North {A}merican
  Chapter of the Association for Computational Linguistics: Human Language
  Technologies, Volume 1 (Long Papers)}, pages 2227--2237, New Orleans,
  Louisiana. Association for Computational Linguistics.

\bibitem[{Piratla et~al.(2020)Piratla, Netrapalli, and
  Sarawagi}]{pmlr-v119-piratla20a}
Vihari Piratla, Praneeth Netrapalli, and Sunita Sarawagi. 2020.
\newblock \href {https://proceedings.mlr.press/v119/piratla20a.html} {Efficient
  domain generalization via common-specific low-rank decomposition}.
\newblock In \emph{Proceedings of the 37th International Conference on Machine
  Learning}, volume 119 of \emph{Proceedings of Machine Learning Research},
  pages 7728--7738. PMLR.

\bibitem[{Pires et~al.(2019)Pires, Schlinger, and
  Garrette}]{pires-etal-2019-multilingual}
Telmo Pires, Eva Schlinger, and Dan Garrette. 2019.
\newblock \href {https://doi.org/10.18653/v1/P19-1493} {How multilingual is
  multilingual {BERT}?}
\newblock In \emph{Proceedings of the 57th Annual Meeting of the Association
  for Computational Linguistics}, pages 4996--5001, Florence, Italy.
  Association for Computational Linguistics.

\bibitem[{Prettenhofer and Stein(2010)}]{prettenhofer-stein-2010-cross}
Peter Prettenhofer and Benno Stein. 2010.
\newblock \href {https://aclanthology.org/P10-1114} {Cross-language text
  classification using structural correspondence learning}.
\newblock In \emph{Proceedings of the 48th Annual Meeting of the Association
  for Computational Linguistics}, pages 1118--1127, Uppsala, Sweden.
  Association for Computational Linguistics.

\bibitem[{Roy et~al.(2020)Roy, Constant, Al-Rfou, Barua, Phillips, and
  Yang}]{roy-etal-2020-lareqa}
Uma Roy, Noah Constant, Rami Al-Rfou, Aditya Barua, Aaron Phillips, and Yinfei
  Yang. 2020.
\newblock \href {https://doi.org/10.18653/v1/2020.emnlp-main.477} {{LAR}e{QA}:
  Language-agnostic answer retrieval from a multilingual pool}.
\newblock In \emph{Proceedings of the 2020 Conference on Empirical Methods in
  Natural Language Processing (EMNLP)}, pages 5919--5930, Online. Association
  for Computational Linguistics.

\bibitem[{Ruder et~al.(2021)Ruder, Constant, Botha, Siddhant, Firat, Fu, Liu,
  Hu, Garrette, Neubig, and Johnson}]{ruder-etal-2021-xtreme}
Sebastian Ruder, Noah Constant, Jan Botha, Aditya Siddhant, Orhan Firat, Jinlan
  Fu, Pengfei Liu, Junjie Hu, Dan Garrette, Graham Neubig, and Melvin Johnson.
  2021.
\newblock \href {https://doi.org/10.18653/v1/2021.emnlp-main.802}
  {{XTREME}-{R}: Towards more challenging and nuanced multilingual evaluation}.
\newblock In \emph{Proceedings of the 2021 Conference on Empirical Methods in
  Natural Language Processing}, pages 10215--10245, Online and Punta Cana,
  Dominican Republic. Association for Computational Linguistics.

\bibitem[{Schmidt(1907)}]{Schmidt1907}
Erhard Schmidt. 1907.
\newblock \href {http://eudml.org/doc/158296} {Zur theorie der linearen und
  nichtlinearen integralgleichungen. i. teil: Entwicklung willkürlicher
  funktionen nach systemen vorgeschriebener}.
\newblock \emph{Mathematische Annalen}, 63:433--476.

\bibitem[{Turk and Pentland(1991)}]{139758}
M.A. Turk and A.P. Pentland. 1991.
\newblock \href {https://doi.org/10.1109/CVPR.1991.139758} {Face recognition
  using eigenfaces}.
\newblock In \emph{Proceedings. 1991 IEEE Computer Society Conference on
  Computer Vision and Pattern Recognition}, pages 586--591.

\bibitem[{Vaswani et~al.(2017)Vaswani, Shazeer, Parmar, Uszkoreit, Jones,
  Gomez, Kaiser, and Polosukhin}]{vaswani2017attention}
Ashish Vaswani, Noam Shazeer, Niki Parmar, Jakob Uszkoreit, Llion Jones,
  Aidan~N Gomez, {\L}ukasz Kaiser, and Illia Polosukhin. 2017.
\newblock Attention is all you need.
\newblock \emph{Advances in neural information processing systems}, 30.

\bibitem[{Wang and Ponce(2021)}]{wang2021a}
Binxu Wang and Carlos~R Ponce. 2021.
\newblock \href {https://openreview.net/forum?id=GH7QRzUDdXG} {A geometric
  analysis of deep generative image models and its applications}.
\newblock In \emph{International Conference on Learning Representations}.

\bibitem[{Wang and Tang(2004)}]{1316855}
Xiaogang Wang and Xiaoou Tang. 2004.
\newblock \href {https://doi.org/10.1109/TPAMI.2004.57} {A unified framework
  for subspace face recognition}.
\newblock \emph{IEEE Transactions on Pattern Analysis and Machine
  Intelligence}, 26(9):1222--1228.

\bibitem[{Wolf et~al.(2020)Wolf, Debut, Sanh, Chaumond, Delangue, Moi, Cistac,
  Rault, Louf, Funtowicz, Davison, Shleifer, von Platen, Ma, Jernite, Plu, Xu,
  Le~Scao, Gugger, Drame, Lhoest, and Rush}]{wolf-etal-2020-transformers}
Thomas Wolf, Lysandre Debut, Victor Sanh, Julien Chaumond, Clement Delangue,
  Anthony Moi, Pierric Cistac, Tim Rault, Remi Louf, Morgan Funtowicz, Joe
  Davison, Sam Shleifer, Patrick von Platen, Clara Ma, Yacine Jernite, Julien
  Plu, Canwen Xu, Teven Le~Scao, Sylvain Gugger, Mariama Drame, Quentin Lhoest,
  and Alexander Rush. 2020.
\newblock \href {https://doi.org/10.18653/v1/2020.emnlp-demos.6} {Transformers:
  State-of-the-art natural language processing}.
\newblock In \emph{Proceedings of the 2020 Conference on Empirical Methods in
  Natural Language Processing: System Demonstrations}, pages 38--45, Online.
  Association for Computational Linguistics.

\bibitem[{Wu and Dredze(2019)}]{wu-dredze-2019-beto}
Shijie Wu and Mark Dredze. 2019.
\newblock \href {https://doi.org/10.18653/v1/D19-1077} {Beto, bentz, becas: The
  surprising cross-lingual effectiveness of {BERT}}.
\newblock In \emph{Proceedings of the 2019 Conference on Empirical Methods in
  Natural Language Processing and the 9th International Joint Conference on
  Natural Language Processing (EMNLP-IJCNLP)}, pages 833--844, Hong Kong,
  China. Association for Computational Linguistics.

\bibitem[{Xu et~al.(2022)Xu, Luo, Chang, Huang, and Huang}]{xu-etal-2022-s4}
Runxin Xu, Fuli Luo, Baobao Chang, Songfang Huang, and Fei Huang. 2022.
\newblock \href {https://aclanthology.org/2022.acl-short.58} {S$^4$-tuning: A
  simple cross-lingual sub-network tuning method}.
\newblock In \emph{Proceedings of the 60th Annual Meeting of the Association
  for Computational Linguistics (Volume 2: Short Papers)}, pages 530--537,
  Dublin, Ireland. Association for Computational Linguistics.

\bibitem[{Yang et~al.(2021)Yang, Yang, Cer, and Darve}]{yang-etal-2021-simple}
Ziyi Yang, Yinfei Yang, Daniel Cer, and Eric Darve. 2021.
\newblock \href {https://doi.org/10.18653/v1/2021.emnlp-main.470} {A simple and
  effective method to eliminate the self language bias in multilingual
  representations}.
\newblock In \emph{Proceedings of the 2021 Conference on Empirical Methods in
  Natural Language Processing}, pages 5825--5832, Online and Punta Cana,
  Dominican Republic. Association for Computational Linguistics.

\bibitem[{Zhao et~al.(2021)Zhao, Eger, Bjerva, and
  Augenstein}]{zhao-etal-2021-inducing}
Wei Zhao, Steffen Eger, Johannes Bjerva, and Isabelle Augenstein. 2021.
\newblock \href {https://doi.org/10.18653/v1/2021.starsem-1.22} {Inducing
  language-agnostic multilingual representations}.
\newblock In \emph{Proceedings of *SEM 2021: The Tenth Joint Conference on
  Lexical and Computational Semantics}, pages 229--240, Online. Association for
  Computational Linguistics.

\bibitem[{Zhu et~al.(2021)Zhu, Feng, Shen, Zhao, Zha, Zhou, and
  Chen}]{NEURIPS2021_8b406655}
Jiapeng Zhu, Ruili Feng, Yujun Shen, Deli Zhao, Zheng-Jun Zha, Jingren Zhou,
  and Qifeng Chen. 2021.
\newblock \href
  {https://proceedings.neurips.cc/paper/2021/file/8b4066554730ddfaa0266346bdc1b202-Paper.pdf}
  {Low-rank subspaces in gans}.
\newblock In \emph{Advances in Neural Information Processing Systems},
  volume~34, pages 16648--16658. Curran Associates, Inc.

\end{thebibliography}
\bibliographystyle{acl_natbib}

\newpage

\appendix
\section{Theoretical Justification}\label{sec:proof}
In this section, we present Theorem~\ref{theorem:objective} and the corresponding proof.
We follow the same proving procedure in \citet{pmlr-v119-piratla20a}.

\begin{theorem}~\label{theorem:objective}
    For any matrix $\boldsymbol{M} \in \mathbb{R}^{d \times L}$, Algorithm~\ref{alg:ours} returns $\boldsymbol{\mu} \in \mathbb{R}^{d}, \boldsymbol{M}_{s} \in \mathbb{R}^{d \times r}, \boldsymbol{\Gamma} \in \mathbb{R}^{L \times r}$ that minimize Equation~\ref{eq:objective} where $\boldsymbol{\mu} \perp \text{Span}\left(\boldsymbol{M}_{s}\right)$.
\end{theorem}

\begin{proof}
    Algorithm~\ref{alg:ours} first obtains the best approximation of $\boldsymbol{M}$ with rank $r + 1$ and $\boldsymbol{\mathbbm{1}}$ in its row space (Line~\ref{line:1}-\ref{line:3}).
    The orthogonal constraint $\boldsymbol{\mu} \perp \text{Span}\left(\boldsymbol{M}_{s}\right)$ is then forced without obeying the low-rank property (Line~\ref{line:4}-\ref{line:5}).
    
    To begin with, note that the optimization problem in Equation~\ref{eq:objective} is equivalent to the following:
    
    \begin{equation}\label{eq:equal}
        \begin{aligned}
        \min_{\widehat{\boldsymbol{M}}} \quad& \left\|\boldsymbol{M}-\widehat{\boldsymbol{M}}\right\|_{F}^{2}\\
        \textrm{s.t.} \quad& \text{rank}\left(\widehat{\boldsymbol{M}}\right) \leq r + 1 \text{ and} \\
        \quad& \boldsymbol{\mathbbm{1}} \in \text{Span}\left(\widehat{\boldsymbol{M}}^\top\right).
        \end{aligned}
    \end{equation}
    
    Let $\boldsymbol{U}, \boldsymbol{\Sigma}, \boldsymbol{V} = \text{SVD} \left(\boldsymbol{M}-\boldsymbol{\mu}^\prime \boldsymbol{\mathbbm{1}}^{\top}\right)$.
    We have that $\boldsymbol{\mathbbm{1}} \perp \text{Span}\left(\boldsymbol{V}^\top\right)$ given $\left(\boldsymbol{M}-\boldsymbol{\mu}^\prime \boldsymbol{\mathbbm{1}}^{\top}\right)\boldsymbol{\mathbbm{1}} = \boldsymbol{0}$.
    Denote by $\boldsymbol{U}_r \boldsymbol{\Sigma}_r \boldsymbol{V}_r^\top$ the top-$r$ component of $\boldsymbol{U} \boldsymbol{\Sigma} \boldsymbol{V}^\top$, by $\sigma_{i}\left(\boldsymbol{A}\right)$ the $i$-th largest singular value of $\boldsymbol{A}$ and by $\boldsymbol{A}_i$ the best rank-$i$ approximation of $\boldsymbol{A}$.
    
    The first step is to show that $\boldsymbol{\mu}^\prime \boldsymbol{\mathbbm{1}}^{\top} + \boldsymbol{U}_r \boldsymbol{\Sigma}_r \boldsymbol{V}_r^\top$ minimizes the objective in Equation~\ref{eq:equal}.
    Following the proof of Eckart-Young-Mirsky theorem for low-rank approximation~\citep{Schmidt1907,eckart1936approximation}, let $\widetilde{\boldsymbol{M}} := \boldsymbol{M}-\widehat{\boldsymbol{M}}$ with any feasible $\widehat{\boldsymbol{M}}$ fixed.
    We have
    \begin{equation*}
        \begin{aligned}
            \sigma_{i}\left(\widetilde{\boldsymbol{M}}\right) 
            =& \left\|\widetilde{\boldsymbol{M}}-\widetilde{\boldsymbol{M}}_{i - 1}\right\|_F \\
            =& \left\|\widetilde{\boldsymbol{M}}-\widetilde{\boldsymbol{M}}_{i - 1}\right\|_F + \left\|\widehat{\boldsymbol{M}}-\widehat{\boldsymbol{M}}\right\|_F \\
            \geq& \left\|\widetilde{\boldsymbol{M}} + \widehat{\boldsymbol{M}} -\widetilde{\boldsymbol{M}}_{i - 1} - \widehat{\boldsymbol{M}}\right\|_F \\
            =& \left\|\boldsymbol{M} -\widetilde{\boldsymbol{M}}_{i - 1} - \widehat{\boldsymbol{M}}\right\|_F \\
            \geq& \min_{\bar{\boldsymbol{M}}} \left\|\boldsymbol{M} -\bar{\boldsymbol{M}}\right\|_F,
        \end{aligned}
    \end{equation*}
    where the minimum is taken over all $\bar{\boldsymbol{M}}$ with $\text{rank}\left(\bar{\boldsymbol{M}}\right) = i + r$ and $\boldsymbol{\mathbbm{1}} \in \text{Span}\left(\bar{\boldsymbol{M}}^\top\right)$.
    By taking $\bar{\boldsymbol{M}} = \boldsymbol{\mu}^\prime \boldsymbol{\mathbbm{1}}^{\top} + \boldsymbol{U}_{i + r - 1} \boldsymbol{\Sigma}_{i + r - 1} \boldsymbol{V}_{i + r - 1}^\top$, we have $\sigma_{i}\left(\widetilde{\boldsymbol{M}}\right) \geq \sigma_{i + r}\left(\boldsymbol{U} \boldsymbol{\Sigma} \boldsymbol{V}^\top\right)$ and therefore $\left\|\boldsymbol{M}-\widehat{\boldsymbol{M}}\right\|_{F}^{2} \geq \left\|\boldsymbol{M}-\boldsymbol{\mu}^\prime \boldsymbol{\mathbbm{1}}^{\top} - \boldsymbol{U}_r \boldsymbol{\Sigma}_r \boldsymbol{V}_r^\top\right\|_{F}^{2}$.
    
    Next, we find $\boldsymbol{\mu}$ and $\boldsymbol{M}_{s}$ that meet the orthogonality constraint while preserving the low-rank structure.
    Suppose $\boldsymbol{\mu} \boldsymbol{\mathbbm{1}}^{\top}+\boldsymbol{M}_{s}{\boldsymbol{\Gamma}}^{\top} = \boldsymbol{\mu}^\prime \boldsymbol{\mathbbm{1}}^{\top}+\boldsymbol{M}_{s}^\prime{\boldsymbol{\Gamma}^\prime}^{\top}$ with $\boldsymbol{\mu} \perp \text{Span}\left(\boldsymbol{M}_{s}\right)$, we have that $\boldsymbol{\mu}^\top \left(\boldsymbol{\mu} \boldsymbol{\mathbbm{1}}^{\top}+\boldsymbol{M}_{s}{\boldsymbol{\Gamma}}^{\top}\right) = \left\|\boldsymbol{\mu}\right\|^2 \boldsymbol{\mathbbm{1}}^\top$ which yields $\boldsymbol{\mu}^\top = \left\|\boldsymbol{\mu}\right\|^2 \left(\boldsymbol{\mu}^\prime \boldsymbol{\mathbbm{1}}^{\top}+\boldsymbol{M}_{s}^\prime{\boldsymbol{\Gamma}^\prime}^{\top}\right)^+ \boldsymbol{\mathbbm{1}}^\top$.
\end{proof}

\section{Base Models}\label{sec:base_models}
We evaluate the alignment methods based on a number of established pretrained multilingual models.
We mainly build on the Transformers library~\citep{wolf-etal-2020-transformers} for our experiments.

\paragraph{mBERT\footnote{\url{https://huggingface.co/bert-base-multilingual-cased}.}}
Multilingual BERT~\citep{devlin-etal-2019-bert} is a transformer model~\citep{vaswani2017attention} pretrained on Wikipedia, with the objective of Masked Language Modeling (MLM) and a shared vocabulary across all languages.

\paragraph{XLM\footnote{\url{https://huggingface.co/xlm-mlm-100-1280}.}} XLM~\citep{NEURIPS2019_c04c19c2} also uses the MLM objective and the monolingual Wikipedia corpus for pretraining, with a larger model and a larger vocabulary.

\paragraph{XLM-R\footnote{\url{https://huggingface.co/xlm-roberta-large}.}} XLM-R
\citep{conneau-etal-2020-unsupervised} follows a similar training procedure as XLM but collects the larger-scale CommonCrawl corpus.

\paragraph{LABSE\footnote{\url{https://huggingface.co/sentence-transformers/LaBSE}.}} LABSE
\citep{feng-etal-2022-language} is the state-of-the-art multilingual sentence encoder that leverages bilingual sentence pairs for pretraining.


Following previous works~\citep{jawahar-etal-2019-bert,ruder-etal-2021-xtreme} that observe certain intermediate layers of Transformer consistently outperform the last layer for cross-lingual tasks, we use the 8th layer for mBERT and XLM, and the 11th layer for XLM-R.
We apply mean-pooling to obtain sentence embeddings as it is widely used~\citep{conneau-etal-2020-emerging, muller-etal-2021-first}.
For LABSE as well as mBERT (X-X) and mBERT (En-En) used in LAReQA, we evaluate the alignment methods on the original sentence embeddings.

\section{Supplementary Results}
In this section, we provide supplementary experimental results.

\label{sec:appendix}

\begin{figure*}[t]
    \centering
    \begin{subfigure}[b]{0.32\linewidth}
         \centering
         \includegraphics[width=\linewidth]{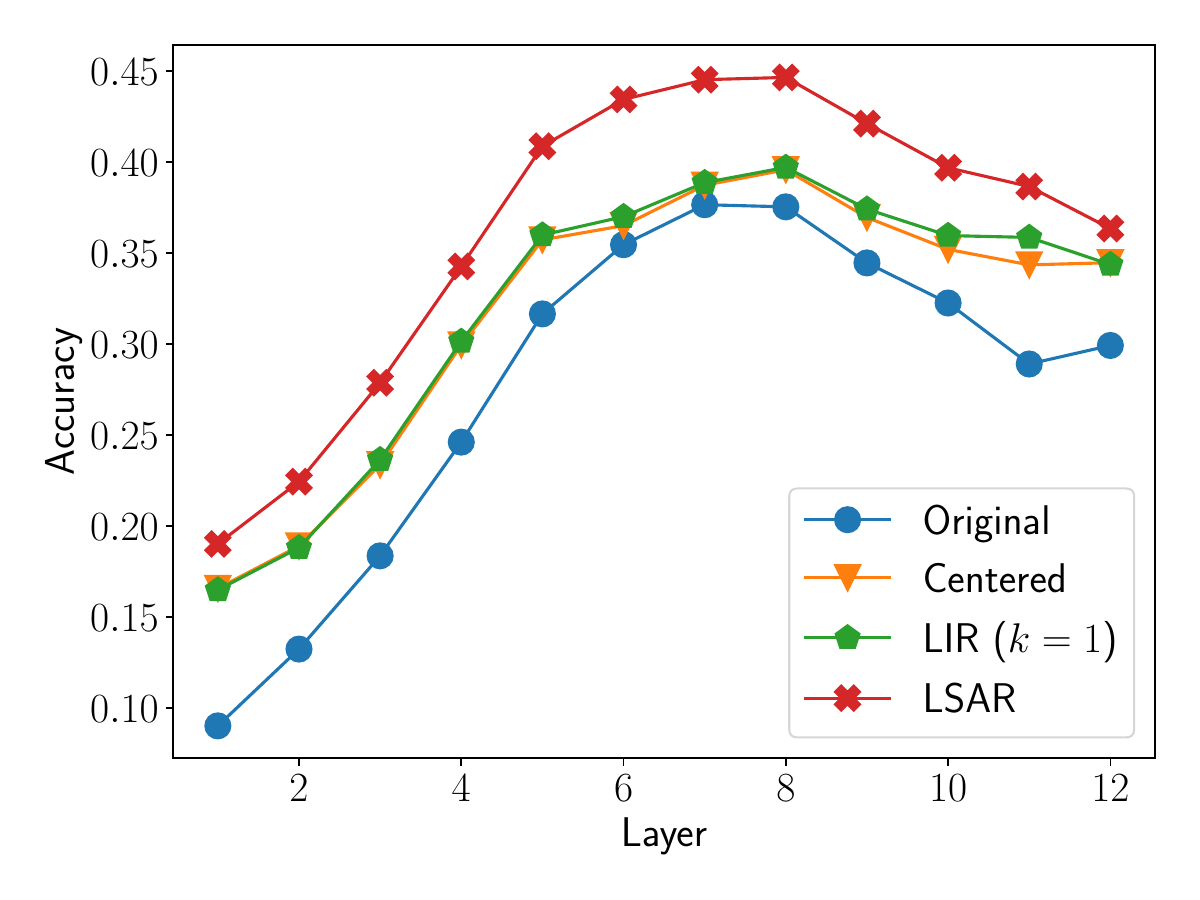}
         \caption{mBERT}
     \end{subfigure}
     \begin{subfigure}[b]{0.32\linewidth}
         \centering
         \includegraphics[width=\linewidth]{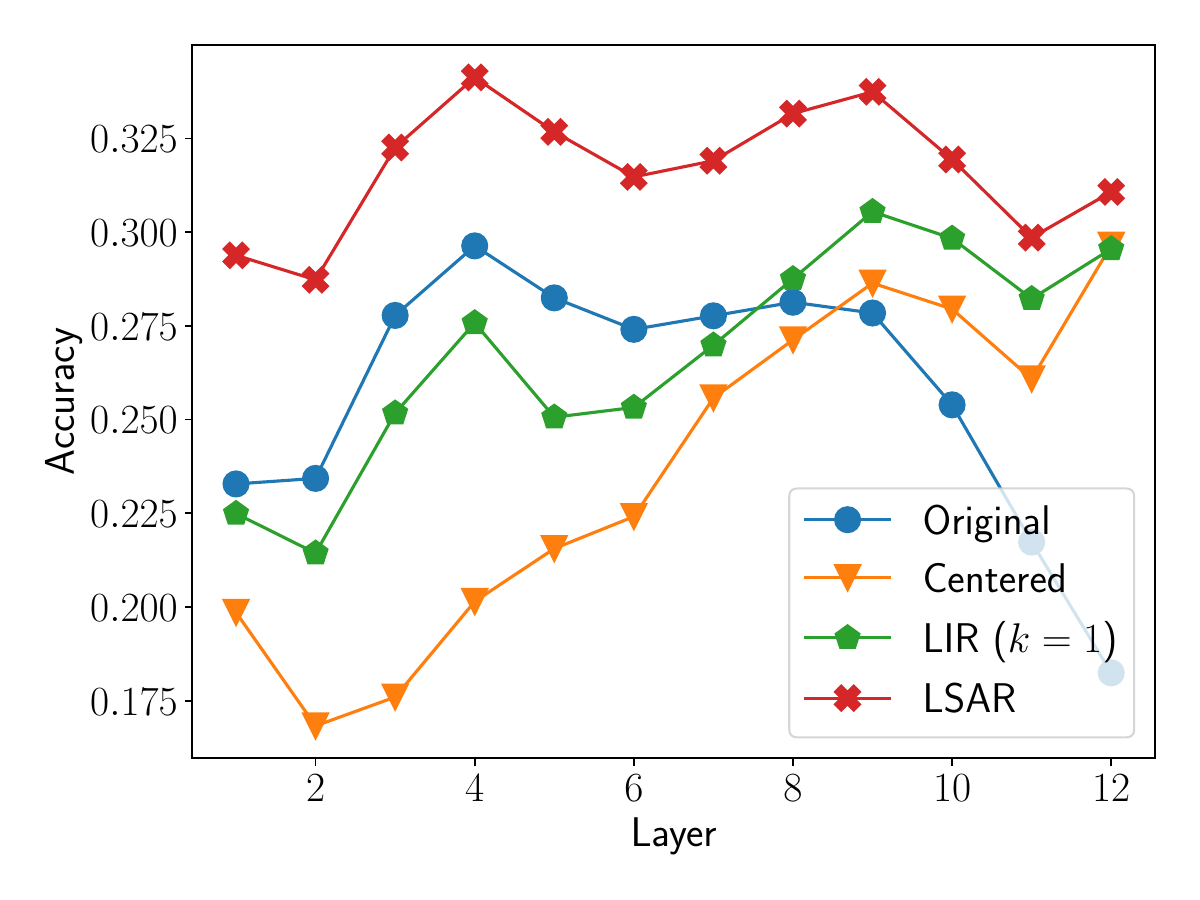}
         \caption{XLM}
     \end{subfigure}
      \begin{subfigure}[b]{0.32\linewidth}
         \centering
         \includegraphics[width=\linewidth]{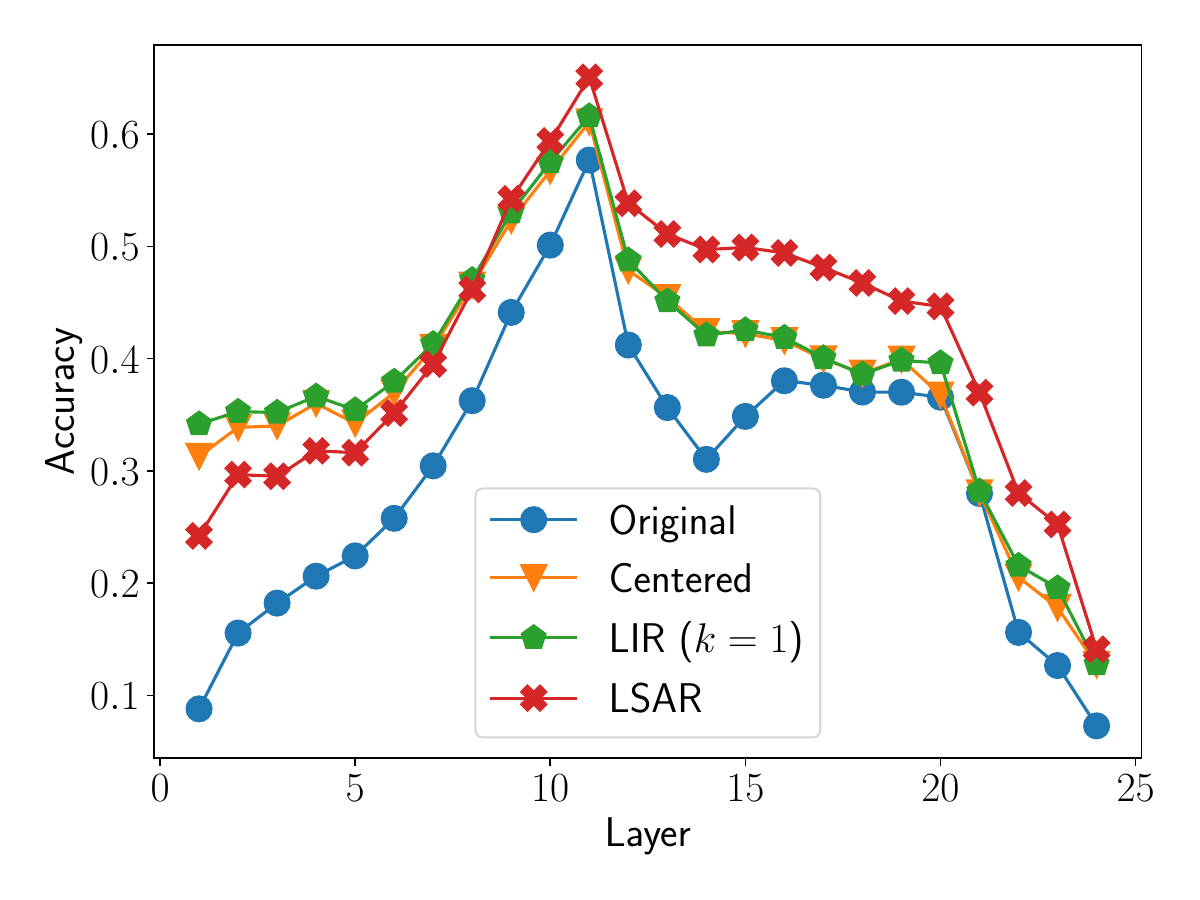}
         \caption{XLM-R}
     \end{subfigure}
    \caption{Retrieval accuracy on Tatoeba (averaged over all 36 languages) at different layers.
    }
    \label{fig:layers}
\end{figure*}

\begin{table}[t]
    \centering
    \begin{tabular}{c|cc|cc}
        \thickhline
        & \multicolumn{2}{c|}{XQuAD-R} & \multicolumn{2}{c}{MLQA-R}\\
        & En-En & X-X & En-En & X-X \\
        \hline
        \original & 28.57 & 23.36 & 35.71 & 26.21\\
        \demean & 35.38 & 45.47 & 35.87 & 43.27 \\
        {\lir} ($k=1$) & 36.71 & 45.24 & 37.56 & 43.24\\
        {\lir} ($k=2$) & 36.70 & 44.74 & 37.11 & 42.42\\
        {\lir} ($k=3$) & 36.82 & 44.54 & 36.87 & 42.28\\
        {\ours} ($r=1$) & 30.51 & 26.38 & 36.79 & 28.79\\
        {\ours} ($r=2$) & 32.31 & 29.22 & 38.05 & 31.70\\
        {\ours} ($r=3$) & 34.05 & 31.99 & 39.00 & 35.28\\
        {\ours} & \textbf{40.95} & \textbf{46.39} & \textbf{40.70} & \textbf{44.02}\\
        \thickhline
    \end{tabular}
    \caption{Answer retrieval mAP (\%) on XQuAD-R and MLQA-R of LAReQA (averaged over all languages), using Wiki-40B as the text resource.
    }
    \label{tab:lareqa_wiki}
\end{table}
\begin{table}[t]
    \centering
    \begin{tabular}{c|cc|cc}
        \thickhline
        & \multicolumn{2}{c|}{XQuAD-R} & \multicolumn{2}{c}{MLQA-R}\\
        & En-En & X-X & En-En & X-X \\
        \hline
        \original & 28.57 & 23.36 & 35.71 & 26.21\\
        \demean & 35.37 & 44.66 & 35.36 & 42.14 \\
        {\lir} ($k=1$) & 37.70 & 44.25 & 38.03 & 41.96\\
        {\lir} ($k=2$) & 36.83 & 43.58 & 37.60 & 41.63\\
        {\lir} ($k=3$) & 36.21 & 43.15 & 36.89 & 41.03\\
        {\ours} ($r=1$) & 30.50 & 26.27 & 36.68 & 28.59\\
        {\ours} ($r=2$) & 32.36 & 28.69 & 37.94 & 31.15\\
        {\ours} ($r=3$) & 34.20 & 31.49 & 38.82 & 34.46\\
        {\ours} & \textbf{41.13} & \textbf{45.89} & \textbf{40.55} & \textbf{43.32}\\
        \thickhline
    \end{tabular}
    \caption{Answer retrieval mAP (\%) on XQuAD-R and MLQA-R of LAReQA (averaged over all languages), using OSCAR as the text resource.
    }
    \label{tab:lareqa_appendix}
\end{table}
\begin{table*}[ht]
    \centering
    \begin{tabular}{c|ccccc|ccccc}
        \thickhline
        & \multicolumn{5}{c|}{Layer 8} & \multicolumn{5}{c}{Layer 12} \\
         & en & de & fr & jp & avg. & em & de & fr & jp & avg. \\
        \hline
        {\original} & 81.13 & 72.82 & 76.02 & 68.98 & 74.74 & 80.07 & 70.05 & 73.75 & 64.86 & 72.18 \\
        {\lir} ($k=1$) & 81.12 & 72.33 & 76.80 & 72.25 & 75.62 & 80.03 & 70.00 & 71.73 & 67.51 & 72.32 \\
        {\lir} ($k=2$) & 81.05 & 71.90 & 76.80 & 72.35 & 75.52 & 79.98 & 71.15 & 72.50 & 69.04 & 73.17 \\
        {\lir} ($k=3$) & 81.10 & 72.23 & 76.22 & 71.06 & 75.15 & 80.03 & 70.85 & 73.67 & 69.36 & 73.48 \\
        {\ours} ($r=1$) & 81.12 & 72.77 & 75.87 & 72.30 & 75.51 & 79.98 & 71.17 & 73.68 & 71.15 & 73.99 \\
        {\ours} ($r=2$) & 81.13 & 72.50 & 76.85 & 72.33 & 75.70 & 80.08 & 71.23 & 73.45 & 70.91 & 73.92 \\
        {\ours} & 81.12 & 72.43 & 76.67 & 72.36 & 75.64 & 79.87 & 70.10 & 71.95 & 68.69 & 72.65 \\
        \thickhline
    \end{tabular}

    \caption{Classification accuracy (\%) on Amazon Reviews (mBERT), using Wiki-40B as the text resource.
    }
    \label{tab:amazon_wiki_mbert}
\end{table*}
\begin{table*}[ht]
    \centering
    \begin{tabular}{c|ccccc|ccccc}
        \thickhline
        & \multicolumn{5}{c|}{Layer 8} & \multicolumn{5}{c}{Layer 12} \\
         & en & de & fr & jp & avg. & em & de & fr & jp & avg. \\
        \hline
        {\original} & 85.45&	69.07&	81.50&	65.21&	75.31& 84.43&	55.42&	72.87&	58.23&	67.74 \\
        {\lir} ($k=1$) & 85.52&	73.68&	81.52&	65.66&	76.59 & 84.50&	75.77&	79.88&	60.98&	75.28 \\
        {\lir} ($k=2$) & 85.52&	73.32&	81.45&	64.31&	76.15 & 84.65&	75.58&	79.73&	60.79&	75.19 \\
        {\lir} ($k=3$) & 85.60&	72.10&	81.62&	62.46&	75.44 & 84.52&	75.52&	79.40&	63.03&	75.62 \\
        {\ours} ($r=1$) & 85.53&	70.98&	81.52&	66.44&	76.12 & 84.45&	56.75&	75.20&	66.64&	70.76 \\
        {\ours} ($r=2$) & 85.48&	73.77&	81.65&	65.43&	76.58 & 84.48&	60.35&	71.25&	66.54&	70.66 \\
        {\ours} & 85.50&	73.78&	81.63&	65.41&	76.58 & 84.48&	75.78&	79.57&	64.99&	76.21 \\
        \thickhline
    \end{tabular}
    \caption{Classification accuracy (\%) on Amazon Reviews (XLM), using Wiki-40B as the text resource.
    }
    \label{tab:amazon_wiki_xlm}
\end{table*}
\begin{table*}[ht]
    \centering
    \begin{tabular}{c|ccccc|ccccc}
        \thickhline
        & \multicolumn{5}{c|}{Layer 11} & \multicolumn{5}{c}{Layer 24} \\
         & en & de & fr & jp & avg. & em & de & fr & jp & avg. \\
        \hline
        {\original} & 84.33&	78.32&	82.30&	76.35&	80.32& 90.55&	78.08&	83.57&	67.14&	79.84 \\
        {\lir} ($k=1$) & 84.33&	82.47&	81.68&	80.18&	82.17 & 90.53&	88.67&	89.88&	86.16&	88.81 \\
        {\lir} ($k=2$) & 84.45&	82.18&	82.10&	80.08&	82.20 & 90.62&	88.48&	88.27&	85.61&	88.25 \\
        {\lir} ($k=3$) & 84.33&	81.40&	83.08&	78.48&	81.82 & 90.67&	88.55&	88.40&	85.61&	88.31 \\
        {\ours} ($r=1$) & 84.35&	77.95&	81.93&	79.78&	81.00 & 90.62&	69.20&	90.00&	83.98&	83.45 \\
        {\ours} ($r=2$) & 84.33&	82.52&	81.17&	80.53&	82.14 & 90.60&	88.73&	79.18&	79.30&	84.45 \\
        {\ours} & 84.30&	82.67&	81.80&	80.56&	82.33 & 90.58&	88.42&	89.33&	85.95&	88.57 \\
        \thickhline
    \end{tabular}
    \caption{Classification accuracy (\%) on Amazon Reviews (XLM-R), using Wiki-40B as the text resource.
    }
    \label{tab:amazon_wiki_xlmr}
\end{table*}
\begin{table*}[ht]
    \centering
        \begin{tabular}{c|ccccc}
        \thickhline
         & en & de & fr & jp & avg.\\
        \hline
        \original & 83.32&
81.37&
84.27&
79.26&
82.05\\
        {\lir} ($k=1$) & 83.18&
81.70&
84.32&
79.51&
82.18 \\
        {\lir} ($k=2$) & 83.20&
81.92&
84.18&
79.33&
82.16 \\
        {\lir} ($k=3$) & 83.18&
81.83&
84.32&
79.45&
82.19 \\
        {\ours} ($r=1$) & 83.32&
81.30&
84.28&
79.21&
82.03 \\
        {\ours} ($r=2$) & 83.10&
81.63&
83.90&
79.61&
82.06 \\
        {\ours} & 83.27&
81.77&
83.95&
79.75&
82.18 \\
        \thickhline
    \end{tabular}
    \caption{Classification accuracy (\%) on Amazon Reviews (LABSE), using Wiki-40B as the text resource.
    }
    \label{tab:amazon_wiki_labse}
\end{table*}

\subsection{Hyperparameter Selection}\label{sec:hyperparameter}
For the considered baselines,
we do not conduct sophisticated hyperparameter search given that it is non-trivial for {\lir}.
To provide fair comparison, for {\lir} and {\ours} that both have one single hyperparameter (the number of top principal components $k$ and the number of basis vectors to span the low-rank subspace $r$), we exhaustively enumerate all values within a scope and report the best performances on the test data.
Figure~\ref{fig:ranks} shows the trend of accuracy on Tatoeba as the hyparameters change.
\begin{figure*}[t]
    \centering
    \begin{subfigure}[b]{0.32\linewidth}
         \centering
         \includegraphics[width=\linewidth]{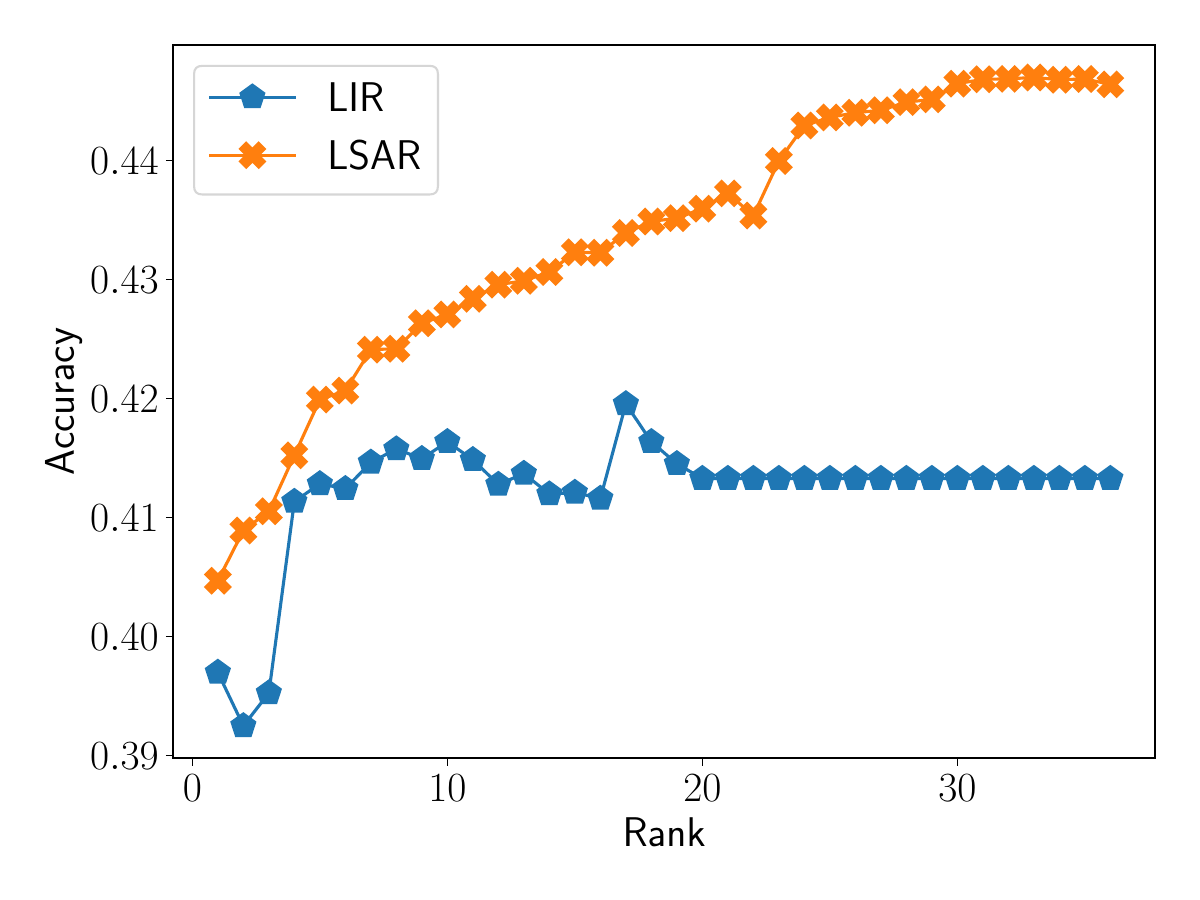}
         \caption{mBERT}
     \end{subfigure}
     \begin{subfigure}[b]{0.32\linewidth}
         \centering
         \includegraphics[width=\linewidth]{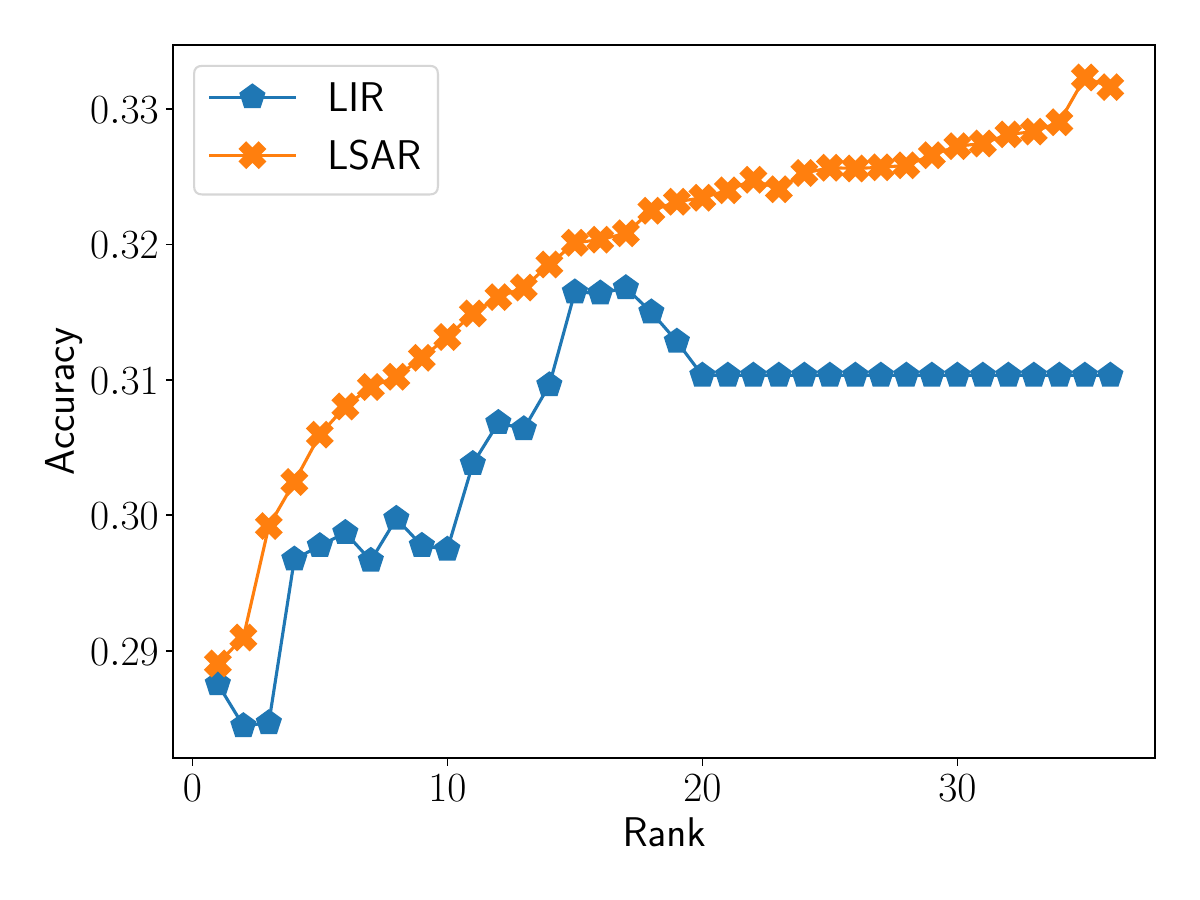}
         \caption{XLM}
     \end{subfigure}
      \begin{subfigure}[b]{0.32\linewidth}
         \centering
         \includegraphics[width=\linewidth]{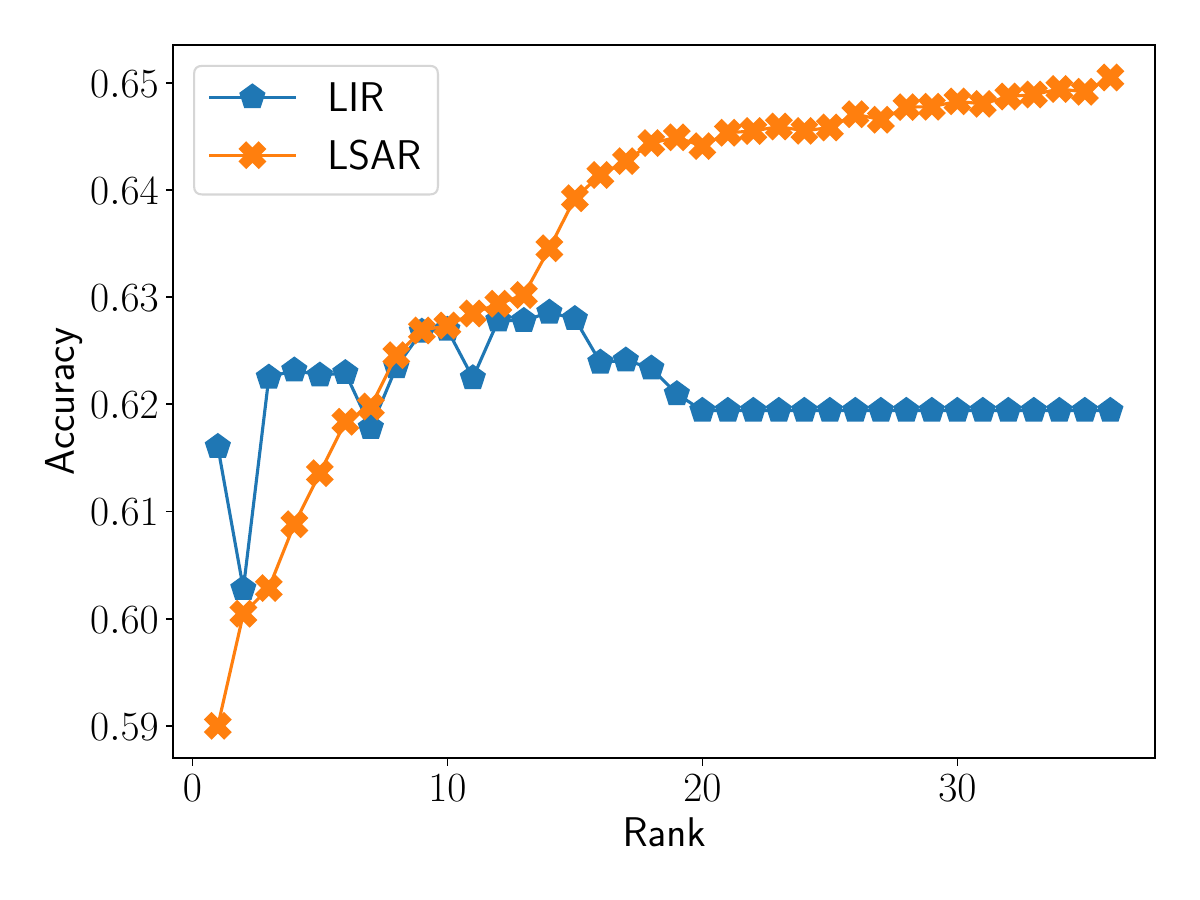}
         \caption{XLM-R}
     \end{subfigure}
    \caption{Retrieval accuracy on Tatoeba (averaged over all 36 languages) with different hyperparameters ($k$ for {\lir} and $r$ for {\ours}).
    We observe that removing more principal components within each language for {\lir} does not result in better performances and can instead lead to information loss.
    For mBERT and XLM, the best $k$ is found 17, whereas it is 14 for XLM-R.
    {\ours}, however, consistently achieves the best results with $r = 36$ as larger subspaces encode more language-specific signals.
    }
    \label{fig:ranks}
\end{figure*}

\subsection{Wiki40-B Results}\label{sec:wiki40b}
In this section we list the results of LAReQA (Table~\ref{tab:lareqa_wiki}) and Amazon Reviews (Table~\ref{tab:amazon_wiki_mbert}-\ref{tab:amazon_wiki_labse}) with Wiki-40B~\citep{guo-etal-2020-wiki}\footnote{\url{https://www.tensorflow.org/datasets/catalog/wiki40b}.} as the text resource.
For Amazon Reviews, we also report the performances obtained in the last layers to reproduce those in \citet{yang-etal-2021-simple}.

For Amazon Reviews, we determine the L2 regularization
strength using a hyperparameter sweep on the 5-fold cross-validation routine, over the range between 1e-4 and 1e4 with 10 logarithmically spaced steps.
This training procedure is implemented using the Scikit-Learn library~\citep{sklearn_api}.

\subsection{OSCAR Results}\label{sec:detail_results}
The detailed results with OSCAR is provided in this section.

\paragraph{Tatoeba}
We report the results for all languages on Tatoeba in Table~\ref{tab:tatoeba_mbert_appendix}-\ref{tab:tatoeba_labse_appendix}.
Additionally, the complete set of results for clustering performance is shown in Table~\ref{tab:cluster_appendix}.

\paragraph{LAReQA} We report the detailed results on LAReQA in Table~\ref{tab:lareqa_appendix}.
We omit listing all languages due to limited space.

\paragraph{Amazon Reviews} We provide the results for all languages on Amazon Reviews in Table~\ref{tab:amazon_oscar_mbert}-\ref{tab:amazon_oscar_labse}.

\begin{table*}[b]
    \centering
    \begin{tabular}{c|cccc}
        \thickhline
        & mBERT & XLM & XLM-R & LABSE \\
        \hline
        {\original} & 0.2815 & 0.5422 & 0.2457 & 0.0344 \\
        {\demean} & 0.0975 & 0.2483 & 0.2004 & 0.0388\\
        {\lir} ($k=1$) & 0.0900 & 0.1875 & 0.2203 & 0.0352 \\
        {\ours} & 0.0801 & 0.1320 & 0.0856 & 0.0306 \\
        \thickhline
    \end{tabular}
    \caption{
    Clustering performance (NMI) of embeddings obtained by mBERT on Tatoeba.
    }
    \label{tab:cluster_appendix}
\end{table*}

\begin{table*}[ht]
    \centering
\begin{tabular}{c|ccccc|ccccc}
        \thickhline
        & \multicolumn{5}{c|}{Layer 8} & \multicolumn{5}{c}{Layer 12} \\
         & en & de & fr & jp & avg. & em & de & fr & jp & avg. \\
        \hline
        {\original} & 81.13 & 72.82 & 76.02 & 68.98 & 74.74 & 80.07 & 70.05 & 73.75 & 64.86 & 72.18 \\
        {\lir} ($k=1$) & 81.12&	72.90&	75.08&	72.43&	75.38 & 80.07&	71.08&	71.40&	67.11&	72.42 \\
        {\lir} ($k=2$) & 81.03&	71.47&	70.58&	66.09&	72.29 & 79.97&	69.35&	72.07&	66.29&	71.92 \\
        {\lir} ($k=3$) & 80.85&	68.67&	74.38&	67.53&	72.86 & 79.88&	66.10&	69.80&	66.59&	70.59 \\
        {\ours} ($r=1$) & 81.25&	72.78&	75.80&	72.48&	75.58 & 79.98&	71.03&	73.62&	70.45&	73.77 \\
        {\ours} ($r=2$) & 81.27&	72.57&	75.85&	72.30&	75.49 & 80.07&	71.12&	73.48&	70.11&	73.69 \\
        {\ours} & 81.15&	72.90&	75.22&	71.68&	75.24 & 79.88&	70.80&	71.70&	67.79&	72.54 \\
        \thickhline
    \end{tabular}

    \caption{Classification accuracy (\%) on Amazon Reviews (mBERT), using OSCAR as the text resource.
    }
    \label{tab:amazon_oscar_mbert}
\end{table*}
\begin{table*}[ht]
    \centering
    \begin{tabular}{c|ccccc|ccccc}
        \thickhline
        & \multicolumn{5}{c|}{Layer 8} & \multicolumn{5}{c}{Layer 12} \\
         & en & de & fr & jp & avg. & em & de & fr & jp & avg. \\
        \hline
        {\original} & 85.45&	69.07&	81.50&	65.21&	75.31& 84.43&	55.42&	72.87&	58.23&	67.74 \\
        {\lir} ($k=1$) & 85.58&
77.57&
80.05&
59.74&
75.74 & 84.52&
75.75&
80.20&
55.26&
73.93 \\
        {\lir} ($k=2$) & 85.40&
76.72&
79.82&
60.86&
75.70 & 84.48&
75.57&
77.95&
55.46&
73.36 \\
        {\lir} ($k=3$) & 85.15&
77.42&
81.07&
51.51&
73.79 & 84.48&
74.55&
76.13&
51.26&
71.61 \\
        {\ours} ($r=1$) & 85.47&
69.08&
81.42&
63.78&
74.94 & 84.50&
56.33&
74.63&
66.84&
70.58 \\
        {\ours} ($r=2$) & 85.37&
74.53&
81.60&
61.88&
75.84 & 84.50&
57.75&
72.80&
66.86&
70.48 \\
        {\ours} & 85.45&
77.15&
80.25&
58.24&
75.27 & 84.62&
75.87&
80.65&
57.14&
74.57 \\
        \thickhline
    \end{tabular}
    \caption{Classification accuracy (\%) on Amazon Reviews (XLM), using OSCAR as the text resource.
    }
    \label{tab:amazon_oscar_xlm}
\end{table*}
\begin{table*}[ht]
    \centering
    \begin{tabular}{c|ccccc|ccccc}
        \thickhline
        & \multicolumn{5}{c|}{Layer 11} & \multicolumn{5}{c}{Layer 24} \\
         & en & de & fr & jp & avg. & em & de & fr & jp & avg. \\
        \hline
        {\original} & 84.33&
78.32&
82.30&
76.35&
80.32&
90.55&
78.08&
83.57&
67.14&
79.84
 \\
        {\lir} ($k=1$) & 84.32&
82.55&
77.82&
79.93&
81.15 & 90.53&
88.85&
87.67&
86.11&
88.29 \\
        {\lir} ($k=2$) & 84.42&
82.27&
78.15&
79.45&
81.07 & 90.63&
89.12&
85.93&
85.86&
87.89 \\
        {\lir} ($k=3$) & 84.33&
81.05&
77.57&
79.16&
80.53 & 90.68&
89.85&
84.68&
86.30&
87.88 \\
        {\ours} ($r=1$) & 84.32&
78.80&
82.12&
80.66&
81.47 & 90.55&
83.47&
77.67&
80.86&
83.14 \\
        {\ours} ($r=2$) & 84.32&
82.55&
82.08&
80.53&
82.37 & 90.55&
87.63&
76.57&
77.66&
83.10 \\
        {\ours} & 84.27&
82.60&
77.85&
80.28&
81.25 & 90.57&
89.37&
88.03&
86.01&
88.50 \\
        \thickhline
    \end{tabular}
    \caption{Classification accuracy (\%) on Amazon Reviews (XLM-R), using OSCAR as the text resource.
    }
    \label{tab:amazon_oscar_xlmr}
\end{table*}
\begin{table*}[ht]
    \centering
        \begin{tabular}{c|ccccc}
        \thickhline
         & en & de & fr & jp & avg.\\
        \hline
        \original & 83.32&
81.37&
84.27&
79.26&
82.05\\
        {\lir} ($k=1$) & 83.40&
81.85&
82.62&
79.81&
81.92 \\
        {\lir} ($k=2$) & 83.28&
80.92&
78.37&
78.73&
80.32 \\
        {\lir} ($k=3$) & 82.88&
78.92&
78.82&
78.85&
79.87\\
        {\ours} ($r=1$) & 83.07&
81.52&
83.88&
79.20&
81.92 \\
        {\ours} ($r=2$) & 83.02&
82.10&
83.55&
79.66&
82.08 \\
        {\ours} & 83.13&
81.92&
83.18&
79.48&
81.93 \\
        \thickhline
    \end{tabular}
    \caption{Classification accuracy (\%) on Amazon Reviews (LABSE), using OSCAR as the text resource.
    }
    \label{tab:amazon_oscar_labse}
\end{table*}

\begin{table*}[t]
    \centering
    \small
    \begin{tabular}{c|cccccccccccc}
        \thickhline
        & af & ar & bg & bn & de & el & es & et & eu & fa & fi & fr  \\
        {\original} & 38.90&	24.50&	48.80&	17.00&	75.40&	29.80&	64.10&	28.10&	25.50&	41.20&	39.00&	64.30\\
        \demean &
        40.90&	27.30&	48.50&	17.30&	74.70&	35.10&	66.40&	29.60&	27.40&	43.70&	40.30&	65.30\\
        {\lir} ($k=1$) & 
        41.00&	27.20&	48.60&	17.90&	74.90&	35.10&	66.40&	30.10&	27.70&	44.00&	40.50&	64.90\\
        \ours & 
        44.70&	31.80&	55.00&	21.90&	79.00&	38.70&	71.20&	35.30&	32.00&	49.80&	46.40&	69.10\\
        \hline
        & he & hi & hu & id & it & ja & jv & ka & kk & ko & ml & mr \\
        {\original} & 40.10&	34.80&	36.90&	53.50&	57.30&	40.90&	17.56&	19.57&	27.13&	36.00&	17.90&	20.10\\
        \demean & 
        41.50&	35.40&	41.40&	53.40&	58.30&	41.60&	18.54&	23.32&	30.96&	38.70&	27.66&	23.00\\
        {\lir} ($k=1$) & 
        41.70&	35.40&	41.60&	53.70&	58.20&	41.90&	18.05&	23.73&	30.96&	38.80&	28.82&	23.00\\
        \ours & 45.70&	43.90&	46.00&	60.00&	61.90&	51.00&	24.88&	28.28&	34.09&	45.30&	36.83&	26.40\\
        \hline
        & nl & pt & ru & sw & ta & te & th & tl & tr & ur & vi & zh \\
        {\original} & 63.70&	68.40&	59.40&	10.77&	13.36&	14.10&	13.69&	16.00&	32.90&	30.80&	61.00&	68.60\\
        \demean & 
        64.30&	69.50&	62.40&	12.56&	14.33&	14.96&	17.15&	18.10&	38.20&	31.40&	62.20&	69.00\\
        {\lir} ($k=1$) & 
        65.10&	69.30&	62.10&	12.31&	14.33&	14.96&	17.15&	18.20&	38.20&	32.10&	62.00&	69.20\\
        \ours & 
        69.20&	73.10&	67.20&	14.36&	18.57&	21.37&	21.72&	22.00&	41.90&	38.00&	67.10&	73.30\\
        \hline
        \thickhline
    \end{tabular}
    \caption{Retrieval accuracy (\%) on Tatoeba for each language (mBERT), using OSCAR as the text resource.
    }
    \label{tab:tatoeba_mbert_appendix}
\end{table*}

\begin{table*}[t]
    \centering
    \small
    \begin{tabular}{c|cccccccccccc}
        \thickhline
        & af & ar & bg & bn & de & el & es & et & eu & fa & fi & fr  \\
        {\original} & 34.20 &
17.80 &
34.80 &
5.70 &
62.20 &
24.90 &
56.00 &
18.40 &
11.90 &
30.50 &
28.10 &
52.80
\\
        \demean & 30.30 &
17.30 &
35.30 &
5.00 &
62.20 &
22.50 &
53.50 &
19.20 &
14.70 &
29.90 &
31.30 &
49.20
\\
        {\lir} ($k=1$) & 32.20 &
18.20 &
37.30 &
5.80 &
65.10 &
25.60 &
54.10 &
21.10 &
16.60 &
31.00 &
32.00 &
51.70
\\
        \ours & 37.50 &
20.10 &
42.40 &
9.90 &
68.20 &
30.50 &
58.80 &
25.50 &
22.00 &
35.00 &
36.10 &
55.10
\\
        \hline
        & he & hi & hu & id & it & ja & jv & ka & kk & ko & ml & mr \\
        {\original} & 31.20 &
15.70 &
29.50 &
44.60 &
52.20 &
32.20 &
19.51 &
22.12 &
14.26 &
25.20 &
0.58 &
6.30\\
        \demean & 30.00 &
14.50 &
30.00 &
45.10 &
49.90 &
28.60 &
17.56 &
19.71 &
14.78 &
22.70 &
0.44 &
5.50\\
        {\lir} ($k=1$) & 31.40 &
17.40 &
31.20 &
45.40 &
50.60 &
31.90 &
19.02 &
21.85 &
16.70 &
24.50 &
0.87 &
6.20\\
        \ours & 34.10 &
24.30 &
36.70 &
49.20 &
55.10 &
36.80 &
22.44 &
24.80 &
20.87 &
29.30 &
4.95 &
10.70\\
        \hline
        & nl & pt & ru & sw & ta & te & th & tl & tr & ur & vi & zh \\
        {\original} & 55.00 &
58.40 &
44.20 &
8.97 &
1.63 &
5.56 &
27.74 &
12.40 &
24.90 &
17.80 &
45.70 &
39.70\\
        \demean & 55.60 &
58.10 &
42.50 &
6.92 &
2.28 &
5.13 &
18.43 &
14.60 &
27.70 &
16.40 &
43.70 &
36.00\\
        {\lir} ($k=1$) & 57.30 &
58.80 &
43.60 &
9.49 &
2.28 &
5.56 &
23.91 &
15.20 &
28.80 &
17.20 &
45.20 &
40.10\\
        \ours & 59.70 &
61.90 &
47.60 &
11.79 &
6.84 &
11.54 &
32.66 &
20.10 &
33.50 &
22.90 &
52.00 &
42.90\\
        \hline
        \thickhline
    \end{tabular}
    \caption{Retrieval accuracy (\%) on Tatoeba for each language (XLM), using OSCAR as the text resource.
    }
    \label{tab:tatoeba_xlm_appendix}
\end{table*}

\begin{table*}[t]
    \centering
    \small
    \begin{tabular}{c|cccccccccccc}
        \thickhline
        & af & ar & bg & bn & de & el & es & et & eu & fa & fi & fr  \\
        {\original} & 58.20 &
47.50 &
71.60 &
43.00 &
88.80 &
61.80 &
75.70 &
52.20 &
35.80 &
70.50 &
71.60 &
73.70\\
        \demean & 59.30 &
49.60 &
75.00 &
45.30 &
90.90 &
65.80 &
76.60 &
57.10 &
45.80 &
72.10 &
78.40 &
73.00\\
        {\lir} ($k=1$) & 59.80 &
50.30 &
75.30 &
46.10 &
90.70 &
66.30 &
77.20 &
57.50 &
47.00 &
72.60 &
78.80 &
73.80\\
        \ours & 65.20 &
55.00 &
76.50 &
52.60 &
91.60 &
71.30 &
80.90 &
60.90 &
52.00 &
75.90 &
78.90 &
77.50\\
        \hline
        & he & hi & hu & id & it & ja & jv & ka & kk & ko & ml & mr \\
        {\original} & 66.40 &
72.20 &
65.40 &
77.00 &
68.30 &
60.60 &
14.15 &
52.28 &
48.52 &
61.40 &
65.36 &
56.80\\
        \demean & 69.10 &
74.10 &
67.90 &
80.00 &
70.60 &
62.50 &
21.95 &
62.60 &
49.57 &
63.00 &
70.01 &
60.30\\
        {\lir} ($k=1$) & 69.50 &
75.00 &
68.20 &
80.40 &
71.50 &
62.60 &
20.98 &
63.00 &
50.78 &
63.50 &
69.87 &
61.20\\
        \ours & 71.80 &
79.40 &
72.70 &
81.50 &
73.70 &
68.20 &
26.34 &
61.53 &
55.65 &
69.70 &
76.71 &
67.60\\
        \hline
        & nl & pt & ru & sw & ta & te & th & tl & tr & ur & vi & zh \\
        {\original} & 80.80 &
82.20 &
74.10 &
20.26 &
26.38 &
35.90 &
29.38 &
36.70 &
65.70 &
23.40 &
74.70 &
68.30\\
        \demean & 81.80 &
81.50 &
78.20 &
24.10 &
30.62 &
41.45 &
30.29 &
37.30 &
74.00 &
26.90 &
79.70 &
72.60\\
        {\lir} ($k=1$) & 82.10 &
82.20 &
78.80 &
25.64 &
31.60 &
41.88 &
31.02 &
37.60 &
74.50 &
27.00 &
80.40 &
73.10\\
        \ours & 84.10 &
84.30 &
79.00 &
26.92 &
36.16 &
44.02 &
35.04 &
47.00 &
75.50 &
32.90 &
79.90 &
73.80\\
        \hline
        \thickhline
    \end{tabular}
    \caption{Retrieval accuracy (\%) on Tatoeba for each language (XLM-R), using OSCAR as the text resource.
    }
    \label{tab:tatoeba_xlmr_appendix}
\end{table*}
\begin{table*}[t]
    \centering
    \small
    \begin{tabular}{c|cccccccccccc}
        \thickhline
        & af & ar & bg & bn & de & el & es & et & eu & fa & fi & fr  \\
        {\original} & 97.70&
90.60&
95.50&
91.60&
99.30&
96.70&
98.10&
98.00&
95.40&
96.30&
97.00&
96.10\\
        \demean & 97.60&
90.40&
95.60&
91.60&
99.30&
96.60&
98.30&
98.10&
95.70&
96.20&
97.20&
96.30\\
        {\lir} ($k=1$) & 97.70&
90.40&
95.60&
91.60&
99.30&
96.80&
98.10&
98.10&
95.80&
96.10&
97.00&
96.30\\
        \ours & 97.40&
90.90&
95.40&
91.60&
99.30&
96.60&
98.20&
97.90&
95.60&
95.90&
97.10&
96.30\\
        \hline
        & he & hi & hu & id & it & ja & jv & ka & kk & ko & ml & mr \\
        {\original} & 92.40&
97.90&
97.00&
95.60&
95.30&
96.40&
85.37&
95.71&
91.13&
94.10&
98.98&
95.00\\
        \demean & 92.10&
97.90&
97.10&
95.80&
95.20&
96.70&
87.80&
95.58&
91.30&
94.20&
99.13&
95.00\\
        {\lir} ($k=1$) & 92.10&
97.90&
97.00&
95.60&
95.40&
96.50&
87.80&
95.71&
91.83&
94.00&
99.13&
95.20\\
        \ours & 92.40&
97.80&
97.10&
95.80&
95.40&
96.50&
85.85&
95.71&
91.65&
93.90&
99.13&
94.80\\
        \hline
        & nl & pt & ru & sw & ta & te & th & tl & tr & ur & vi & zh \\
        {\original} & 97.50&
95.70&
95.30&
89.49&
90.23&
98.29&
97.08&
98.00&
98.20&
96.00&
97.80&
96.10\\
        \demean & 97.70&
95.60&
95.00&
89.23&
90.23&
98.72&
97.26&
97.90&
98.20&
95.70&
97.90&
96.00\\
        {\lir} ($k=1$) & 97.70&
95.70&
95.20&
90.26&
90.55&
98.72&
97.45&
98.00&
98.30&
95.90&
97.80&
96.00\\
        \ours & 97.50&
96.00&
95.40&
90.26&
90.55&
98.72&
96.90&
97.80&
98.30&
96.00&
97.70&
96.20\\
        \hline
        \thickhline
    \end{tabular}
    \caption{Retrieval accuracy (\%) on Tatoeba for each language (LABSE), using OSCAR as the text resource.
    }
    \label{tab:tatoeba_labse_appendix}
\end{table*}


\end{document}